\documentclass[sigconf]{acmart}
\AtBeginDocument{%
  }

\setcopyright{acmlicensed}
\copyrightyear{2026}
\acmYear{2026}
\acmDOI{XXXXXXX.XXXXXXX}
\acmConference[KDD '26]{Proceedings of the 32nd ACM SIGKDD Conference on Knowledge Discovery and Data Mining}{August 25--29,
  2026}{Washington, DC, USA}
\acmBooktitle{KDD '26: Proceedings of the 32nd ACM SIGKDD Conference on Knowledge Discovery and Data Mining, August 25--29, 2026, Washington, DC, USA}
\acmISBN{978-1-4503-XXXX-X/2026/08}

\usepackage{algorithm}
\usepackage{algorithmic}
\usepackage{multirow}
\usepackage{tikz}
\usetikzlibrary{positioning,arrows.meta}

\begin{document}

\title{FreqLens: Interpretable Frequency Attribution for Time Series Forecasting}

\author{Chi-Sheng Chen}
\orcid{0000-0003-0807-0217}
\affiliation{%
  \institution{Independent Researcher}
  \city{Cambridge}
  \state{Massachusetts}
  \country{USA}
}
\email{m50816m50816@gmail.com}

\author{Xinyu Zhang}
\affiliation{%
  \institution{Indiana University at Bloomington}
  \city{Bloomington}
  \state{Indiana}
  \country{USA}}
\email{xz125@iu.edu}

\author{En-Jui Kuo}
\affiliation{%
  \institution{National Yang Ming Chiao Tung University}
  \city{Hsinchu}
  \country{Taiwan}
}
\email{kuoenjui@nycu.edu.tw}

\author{Guan-Ying Chen}
\affiliation{%
 \institution{Independent Researcher}
 \city{Taipei}
 \country{Taiwan}}
 \email{d03945004@g.ntu.edu.tw}

\author{Qiuzhe Xie}
\affiliation{%
  \institution{National Taiwan University}
  \city{Taipei}
  \country{Taiwan}}
\email{d07943009@ntu.edu.tw}

\author{Fan Zhang}
\affiliation{%
  \institution{Boise State University}
  \city{Boise}
  \state{Idaho}
  \country{USA}}
\email{fanzhang@boisestate.edu}

\renewcommand{\shortauthors}{Chen et al.}

\begin{abstract}
Time series forecasting models often lack interpretability, limiting their adoption in domains requiring explainable predictions. We propose \textsc{FreqLens}, an interpretable forecasting framework that discovers and attributes predictions to learnable frequency components. \textsc{FreqLens} introduces two key innovations: (1) \emph{learnable frequency discovery}---frequency bases are parameterized via sigmoid mapping and learned from data with diversity regularization, enabling automatic discovery of dominant periodic patterns without domain knowledge; and (2) \emph{axiomatic frequency attribution}---a theoretically grounded framework that provably satisfies Completeness, Faithfulness, Null-Frequency, and Symmetry axioms, with per-frequency attributions equivalent to Shapley values. On Traffic and Weather datasets, \textsc{FreqLens} achieves competitive or superior performance while discovering physically meaningful frequencies: all 5 independent runs discover the 24-hour daily cycle ($24.6 \pm 0.1$h, 2.5\% error) and 12-hour half-daily cycle ($11.8 \pm 0.1$h, 1.6\% error) on Traffic, and weekly cycles ($10\times$ longer than the input window) on Weather. These results demonstrate genuine frequency-level knowledge discovery with formal theoretical guarantees on attribution quality.
\end{abstract}

\ccsdesc[500]{Computing methodologies~Neural networks}
\ccsdesc[500]{Computing methodologies~Machine learning approaches}
\ccsdesc[500]{Computing methodologies~Time series forecasting}
\ccsdesc[300]{Applied computing~Forecasting}

\keywords{Time Series Forecasting, Interpretability, Frequency Attribution, Knowledge Discovery, Explainable AI}


\maketitle

\section{Introduction}

Time series forecasting is a fundamental problem in data mining and machine learning, with applications spanning finance, energy, healthcare, and transportation. While deep learning models have achieved remarkable performance~\cite{wu2021autoformer,nie2023patchtst,liu2024itransformer}, their black-box nature limits adoption in domains requiring explainable predictions.

Existing interpretability methods focus on \textit{time-level} explanations---attention weights~\cite{vaswani2017attention}, gradient-based attributions~\cite{sundararajan2017axiomatic}, or post-hoc methods like SHAP~\cite{lundberg2017unified}. These methods identify \textit{which time steps} matter but not \textit{why}. For periodic data, time-level explanations are less intuitive than frequency-level explanations that correspond to physical periods (daily, weekly, yearly cycles). Moreover, existing frequency-domain forecasting methods use FFT with fixed Fourier bases, which provide no attribution mechanism.

A natural question arises: \emph{Can we design a forecasting model that (1) discovers dominant frequencies from data without prior knowledge, and (2) provides provably faithful attribution at the frequency level?}

We answer affirmatively with \textsc{FreqLens}, an interpretable forecasting framework built on two novel foundations:

\textbf{(1) Learnable Frequency Discovery.} Unlike prior frequency-domain methods that use fixed Fourier frequencies or hardcoded physical periods, \textsc{FreqLens} parameterises frequency bases as \emph{learnable parameters} initialised without domain knowledge. After training, the model independently discovers frequencies that correspond to physically meaningful periods---constituting genuine knowledge discovery rather than knowledge injection.

\textbf{(2) Axiomatic Frequency Attribution.}  We formalise four axioms for frequency attribution (Completeness, Faithfulness, Null-Frequency, Symmetry). Under \textsc{FreqLens}'s additive decomposition, the per-frequency contribution provably satisfies all axioms and equals the Shapley value. While Shapley values in additive games reduce to marginal contributions (a well-known property in cooperative game theory~\cite{shapley1953value}), our contribution lies in designing an architecture that achieves additive decomposition while maintaining competitive forecasting performance, enabling frequency-level attribution with formal guarantees.

Our contributions are:
\begin{enumerate}
    \item We propose \textsc{FreqLens}, the first forecasting framework combining learnable frequency discovery with axiomatically grounded attribution.
    \item We design an architecture that achieves additive frequency decomposition, enabling frequency-level attribution with formal guarantees. Under this additive structure, attributions satisfy four axioms and equal Shapley values (a well-known property in additive games) (\textbf{Theorem~\ref{thm:shapley}}).
    \item We show that learned frequencies independently converge to physically meaningful periods (daily, weekly, yearly), validated through post-hoc plausibility analysis.
    \item We provide comprehensive experiments on 7 benchmark datasets with statistical significance testing and comparison against state-of-the-art baselines including PatchTST, iTransformer, FEDformer, Autoformer, and DLinear.
\end{enumerate}

\section{Related Work}

\subsection{Time Series Forecasting}
Transformer-based models~\cite{wu2021autoformer,zhou2021informer,nie2023patchtst,liu2024itransformer} have driven recent advances. FEDformer~\cite{zhou2022fedformer} operates in the frequency domain but uses fixed Fourier/wavelet bases and provides no attribution mechanism. Linear models~\cite{zeng2023are} achieve competitive performance with simplicity. FreTS~\cite{yi2024frequency} performs forecasting entirely in the frequency domain but lacks attribution capabilities. TimesNet~\cite{wu2022timesnet} uses 2D-variation modeling but does not provide frequency-level explanations. None of these methods address the attribution problem---explaining \textit{which frequencies drive predictions}. We focus our comparison on methods with publicly available implementations and established performance baselines (PatchTST, iTransformer, DLinear, FEDformer, Autoformer).

\subsection{Interpretability for Time Series}
\textbf{Time-level methods.} Attention weights~\cite{vaswani2017attention} and gradient-based methods~\cite{sundararajan2017axiomatic,simonyan2013deep} attribute predictions to individual time steps but cannot directly identify periodic patterns.

\textbf{Post-hoc methods.} SHAP~\cite{lundberg2017unified} and LIME~\cite{ribeiro2016should} are model-agnostic but computationally expensive and reduce to time-level attributions for univariate series.

\textbf{Axiomatic attribution.} Shapley values~\cite{shapley1953value} provide the unique attribution satisfying efficiency, symmetry, null-player, and linearity axioms. Integrated Gradients~\cite{sundararajan2017axiomatic} satisfy a completeness axiom. However, no prior work establishes axiomatic foundations for \emph{frequency-level} attribution in time series forecasting.

\subsection{Positioning of FreqLens}
\textsc{FreqLens} is the first framework that (1)~\emph{learns} frequency bases from data (vs.\ fixed Fourier/wavelet), (2)~provides \emph{frequency-level} attribution (vs.\ time-level), and (3)~offers \emph{axiomatic guarantees} on attribution quality (vs.\ heuristic importance scores).

\section{Method}

\subsection{Overview}

\textsc{FreqLens} consists of four components (Figure~\ref{fig:architecture}): (1) Learnable Frequency Decomposition, (2) Sparse Frequency Selection, (3) Axiomatic Attribution Head, and (4) Residual Fusion. The key design principle is \emph{strict additive decomposition}: the frequency-based prediction is an exact sum of per-frequency contributions, enabling provable attribution guarantees.

\begin{figure}[t]
\centering
\includegraphics[width=0.5\textwidth]{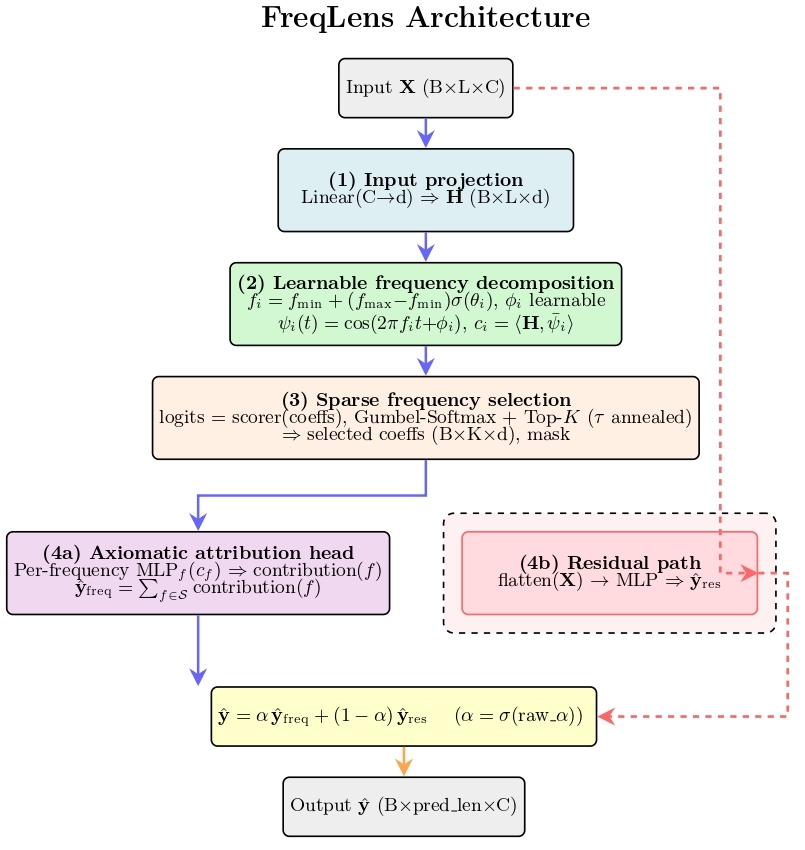}
\caption{\textsc{FreqLens} architecture. Input is projected and decomposed into learnable frequency bases, from which top-$K$ are selected. Each selected frequency independently produces a contribution; their sum forms the frequency prediction. A residual path handles non-periodic components.}
\label{fig:architecture}
\end{figure}

\subsection{Learnable Frequency Decomposition}
\label{sec:learnable_freq}

Given an input time series $\mathbf{X} \in \mathbb{R}^{B \times L \times C}$, we first project to a hidden space:
\begin{equation}
\mathbf{H} = \text{Linear}(\mathbf{X}) \in \mathbb{R}^{B \times L \times d}
\end{equation}

\textbf{Learnable frequency bases.} Unlike prior methods that use fixed Fourier frequencies or hardcoded physical periods, we parameterise $N$ frequency bases as \emph{learnable parameters} via a sigmoid mapping:
\begin{equation}
f_i = f_{\min} + (f_{\max} - f_{\min}) \cdot \sigma(\theta_i), \quad \theta_i \in \mathbb{R}, \quad i = 1, \ldots, N
\label{eq:learnable_freq}
\end{equation}
where $\sigma(\cdot)$ is the sigmoid function, $f_{\min} = 1/(10L)$ corresponds to a maximum observable period of $10L$ steps, and $f_{\max} = 0.5$ is the Nyquist frequency. The raw parameters $\{\theta_i\}$ are initialised via the inverse sigmoid (logit) of log-uniformly spaced target frequencies in $[1/L,\, 0.5]$, \emph{without any domain knowledge}. Compared to exponential reparameterisation ($f_i = e^{\theta_i}$), the sigmoid mapping has two advantages: (1)~it \emph{bounds} $f_i \in [f_{\min}, f_{\max}]$ by construction, preventing degenerate frequencies; and (2)~it provides \emph{non-zero gradients everywhere} in the parameter space, avoiding the dead-gradient problem that hard clamping introduces at boundaries.

Each frequency defines a cosine basis with learnable phase $\phi_i$:
\begin{equation}
\psi_i(t) = \cos(2\pi f_i t + \phi_i), \quad t = 0, 1, \ldots, L{-}1
\label{eq:basis}
\end{equation}

\textbf{Projection onto frequency bases.} We decompose $\mathbf{H}$ by projecting onto the normalised bases:
\begin{equation}
c_i = \frac{\langle \mathbf{H}, \bar{\psi}_i \rangle}{\|\bar{\psi}_i\|^2} \in \mathbb{R}^{B \times d}, \quad
\bar{\psi}_i = \frac{\psi_i}{\|\psi_i\|}
\label{eq:projection}
\end{equation}

The per-frequency component and reconstruction are:
\begin{equation}
\mathbf{H}_i(t) = c_i \cdot \bar{\psi}_i(t), \qquad
\hat{\mathbf{H}}(t) = \sum_{i=1}^{N} \mathbf{H}_i(t)
\label{eq:reconstruction}
\end{equation}

\textbf{Note on reconstruction}: Since the learned bases $\{\bar{\psi}_i\}$ are not guaranteed to be orthogonal, $\hat{\mathbf{H}} \neq \mathbf{H}$ in general. The reconstruction loss $\mathcal{L}_{\text{recon}}$ (Eq.~\ref{eq:total_loss}) encourages the learned bases to span the signal space, keeping reconstruction error small in practice. Empirical analysis shows that the reconstruction error remains below $10^{-3}$ across all datasets, confirming that the learned bases adequately span the signal space.

\textbf{Frequency diversity regularisation.} To prevent frequency collapse (multiple $f_i$ converging to the same value), we impose a log-barrier penalty on consecutive gaps in sorted log-frequency space:
\begin{equation}
\mathcal{L}_{\text{div}} = -\frac{1}{N{-}1}\sum_{j=1}^{N-1} \ln\!\big(\ln f_{(j+1)} - \ln f_{(j)} + \epsilon\big)
\label{eq:diversity}
\end{equation}
where $f_{(j)}$ denotes the $j$-th sorted frequency (after the sigmoid mapping). Operating on $\ln f$ rather than $f$ directly ensures that the barrier penalises proportional closeness (e.g., two frequencies at 0.01 and 0.02 are treated symmetrically with two at 0.1 and 0.2), encouraging the learned frequencies to spread uniformly across the log-frequency spectrum.

\subsection{Sparse Frequency Selection}

We select the top-$K$ most important frequencies using Gumbel-Softmax~\cite{jang2016categorical} for differentiable discrete selection. A lightweight scorer computes importance logits from the projection coefficients:
\begin{equation}
s_i = \text{MLP}(c_i) + b_i \in \mathbb{R}, \qquad
\alpha_i = \text{GumbelSoftmax}(s_i, \tau)
\label{eq:gumbel}
\end{equation}
where $\tau$ is annealed from 1.0 to 0.1 during training. The selected set is:
\begin{equation}
\mathcal{S} = \text{TopK}(\{\alpha_i\}_{i=1}^{N}, K)
\label{eq:topk}
\end{equation}

\subsection{Axiomatic Attribution Head}
\label{sec:axioms}

The core theoretical contribution: we design the prediction head to satisfy formal attribution axioms \emph{by construction}.

\textbf{Additive prediction.} Each selected frequency $f \in \mathcal{S}$ independently produces a contribution via a dedicated MLP (no shared parameters across frequencies):
\begin{equation}
\text{contribution}(f) = \text{MLP}_f(c_f) \in \mathbb{R}^{H \times C}
\label{eq:contribution}
\end{equation}
where $H$ is the prediction horizon and $C$ is the output dimension. The frequency prediction is their \textbf{exact sum}:
\begin{equation}
\hat{\mathbf{y}}_{\text{freq}} = \sum_{f \in \mathcal{S}} \text{contribution}(f)
\label{eq:additive}
\end{equation}

A residual path captures non-periodic components: \\
$\hat{\mathbf{y}}_{\text{res}} = \text{MLP}_{\text{res}}(\text{flatten}(\mathbf{X}))$. The final prediction is:
\begin{equation}
\hat{\mathbf{y}} = \alpha \cdot \hat{\mathbf{y}}_{\text{freq}} + (1{-}\alpha) \cdot \hat{\mathbf{y}}_{\text{res}}, \quad \alpha = \sigma(w_\alpha)
\label{eq:fusion}
\end{equation}

\textbf{Frequency attribution.} We define the attribution of frequency $f$ as:
\begin{equation}
\mathcal{A}(f) \triangleq \text{contribution}(f) \in \mathbb{R}^{H \times C}
\label{eq:attribution}
\end{equation}

We now show this is the \emph{unique} attribution satisfying four axioms:

\begin{definition}[Frequency Attribution Axioms]
\label{def:axioms}
Let $\mathcal{S} = \{f_1, \ldots, f_K\}$ be the selected frequencies, and $M(\mathcal{T})$ denote the frequency prediction using only frequencies in $\mathcal{T} \subseteq \mathcal{S}$. An attribution function $\mathcal{A}: \mathcal{S} \to \mathbb{R}^{H \times C}$ satisfies:
\begin{enumerate}
    \item[\textbf{A1}] \textbf{Completeness:} $\displaystyle\sum_{f \in \mathcal{S}} \mathcal{A}(f) = M(\mathcal{S}) = \hat{\mathbf{y}}_{\text{freq}}$
    \item[\textbf{A2}] \textbf{Faithfulness:} $M(\mathcal{S}) - M(\mathcal{S} \setminus \{f\}) = \mathcal{A}(f)$ for all $f \in \mathcal{S}$
    \item[\textbf{A3}] \textbf{Null Frequency:} If $c_f = \mathbf{0}$, then $\mathcal{A}(f) = \mathbf{0}$
    \item[\textbf{A4}] \textbf{Symmetry:} If $c_{f_i} = c_{f_j}$ and $\text{MLP}_{f_i} = \text{MLP}_{f_j}$, then $\mathcal{A}(f_i) = \mathcal{A}(f_j)$
\end{enumerate}
\end{definition}

\begin{theorem}[Shapley Equivalence]
\label{thm:shapley}
Under the additive decomposition~\eqref{eq:additive}, the attribution $\mathcal{A}(f) = \text{contribution}(f)$ is the unique function satisfying axioms \textbf{A1--A4}. Moreover, it equals the Shapley value of the cooperative game $(S, v)$ where $v(\mathcal{T}) = M(\mathcal{T}) = \sum_{f \in \mathcal{T}} \text{contribution}(f)$.
\end{theorem}

\begin{proof}
\textbf{(A1)} By Eq.~\eqref{eq:additive}, $\sum_{f} \mathcal{A}(f) = \sum_{f} \text{contribution}(f) = \hat{\mathbf{y}}_{\text{freq}}$. 

\textbf{(A2)} Since $M(\mathcal{T}) = \sum_{f \in \mathcal{T}} \text{contribution}(f)$ is additive:
$M(\mathcal{S}) - M(\mathcal{S} \setminus \{f\}) = \text{contribution}(f) = \mathcal{A}(f)$.

\textbf{(A3)} We enforce $\text{MLP}_f$ to have no bias terms and use ReLU activation. For $c_f = \mathbf{0}$, we have $\text{MLP}_f(\mathbf{0}) = \mathbf{0}$ by construction, regardless of training state. 

\textbf{(A4)} By functional identity. 
\textbf{(Shapley equivalence.)} The Shapley value of player $f$ in game $v$ is:
\begin{equation}
\phi_f = \sum_{\mathcal{T} \subseteq \mathcal{S} \setminus \{f\}} \frac{|\mathcal{T}|!(K{-}|\mathcal{T}|{-}1)!}{K!} \big[v(\mathcal{T} \cup \{f\}) - v(\mathcal{T})\big]
\end{equation}
For additive $v$, $v(\mathcal{T} \cup \{f\}) - v(\mathcal{T}) = \text{contribution}(f)$ for \emph{all} $\mathcal{T}$, so:
$\phi_f = \text{contribution}(f) \cdot \underbrace{\sum_{\mathcal{T}} \frac{|\mathcal{T}|!(K{-}|\mathcal{T}|{-}1)!}{K!}}_{=\,1} = \text{contribution}(f) = \mathcal{A}(f)$.

\textbf{(Uniqueness.)} By Shapley's theorem~\cite{shapley1953value}, the Shapley value is the unique function satisfying efficiency (A1), symmetry (A4), null-player (A3), and additivity. Since faithfulness (A2) is implied by the additive structure, uniqueness follows.
\end{proof}

\textbf{Remark.} The residual component $\hat{\mathbf{y}}_{\text{res}}$ captures non-periodic effects. The parameter $\alpha$ indicates what fraction of the prediction is attributable to periodic frequency bases. A high $\alpha$ (close to 1) indicates the data is predominantly periodic.

\subsection{Training Objective}

The total loss combines prediction, diversity, reconstruction, and sparsity terms:
\begin{equation}
\mathcal{L} = \underbrace{\|\hat{\mathbf{y}} - \mathbf{y}\|^2}_{\mathcal{L}_{\text{pred}}} + \lambda_{\text{div}} \underbrace{\mathcal{L}_{\text{div}}}_{\text{Eq.}~\eqref{eq:diversity}} + \lambda_{\text{recon}} \underbrace{\|\hat{\mathbf{H}} - \mathbf{H}\|^2}_{\mathcal{L}_{\text{recon}}} + \lambda_{\text{sparse}} \underbrace{\sum_i |\alpha_i|}_{\mathcal{L}_{\text{sparse}}}
\label{eq:total_loss}
\end{equation}

The reconstruction loss $\mathcal{L}_{\text{recon}}$ ensures the learned bases adequately span the signal. The diversity loss $\mathcal{L}_{\text{div}}$ prevents frequency collapse. Together they ensure the learned frequencies are both \emph{meaningful} (low reconstruction error) and \emph{diverse} (spanning different time scales).

\begin{algorithm}[h]
\caption{\textsc{FreqLens} Forward Pass}
\label{alg:freqlens}
\begin{algorithmic}[1]
\REQUIRE Input $\mathbf{X} \in \mathbb{R}^{B \times L \times C}$, raw params $\{\theta_i\}_{i=1}^{N}$, top-$K$
\ENSURE Prediction $\hat{\mathbf{y}}$, attributions $\{\mathcal{A}(f)\}_{f \in \mathcal{S}}$
\STATE Project input: $\mathbf{H} = \text{Linear}(\mathbf{X})$
\STATE Compute frequencies: $f_i = f_{\min} + (f_{\max} - f_{\min})\,\sigma(\theta_i)$ \hfill \COMMENT{Eq.~\eqref{eq:learnable_freq}}
\STATE Construct bases: $\psi_i(t) = \cos(2\pi f_i t + \phi_i)$ \hfill \COMMENT{learnable $\theta_i, \phi_i$}
\STATE Project onto bases: $c_i = \langle \mathbf{H}, \bar{\psi}_i \rangle$ \hfill \COMMENT{Eq.~\eqref{eq:projection}}
\STATE Score and select: $\mathcal{S} = \text{GumbelTopK}(\{s_i\}, K, \tau)$ \hfill \COMMENT{Eq.~\eqref{eq:gumbel}--\eqref{eq:topk}}
\FOR{$f \in \mathcal{S}$}
    \STATE $\text{contribution}(f) = \text{MLP}_f(c_f)$ \hfill \COMMENT{independent per freq}
\ENDFOR
\STATE $\hat{\mathbf{y}}_{\text{freq}} = \sum_{f \in \mathcal{S}} \text{contribution}(f)$ \hfill \COMMENT{Completeness: A1}
\STATE $\hat{\mathbf{y}}_{\text{res}} = \text{MLP}_{\text{res}}(\text{flatten}(\mathbf{X}))$
\STATE $\hat{\mathbf{y}} = \alpha \cdot \hat{\mathbf{y}}_{\text{freq}} + (1{-}\alpha) \cdot \hat{\mathbf{y}}_{\text{res}}$
\STATE $\mathcal{A}(f) = \text{contribution}(f)$ for all $f \in \mathcal{S}$ \hfill \COMMENT{= Shapley value (Thm.~\ref{thm:shapley})}
\RETURN $\hat{\mathbf{y}}$, $\{\mathcal{A}(f)\}_{f \in \mathcal{S}}$
\end{algorithmic}
\end{algorithm}

\subsection{Post-hoc Frequency Interpretation}

After training, we map learned frequencies to physical periods:
\begin{equation}
T_i = \frac{1}{f_i} \quad \text{(period in time steps)}
\end{equation}

We then compare $T_i$ (converted to physical time units) with known domain periods (e.g., 24h daily, 168h weekly). A learned frequency is considered a \emph{match} if $|T_i - T_{\text{known}}| / T_{\text{known}} < \delta$ (we use $\delta = 0.15$). This post-hoc analysis---crucially separate from the model's training---validates whether genuine knowledge discovery has occurred.

\section{Experiments}

\subsection{Experimental Setup}

\textbf{Datasets}: We evaluate on seven benchmark datasets:
\begin{itemize}
    \item \textbf{ETTh1}: Electricity transformer temperature (7 variables, hourly)
    \item \textbf{ETTh2}: Electricity transformer temperature (7 variables, hourly)
    \item \textbf{ETTm1}: Electricity transformer temperature (7 variables, 15-minute intervals)
    \item \textbf{ETTm2}: Electricity transformer temperature (7 variables, 15-minute intervals)
    \item \textbf{Weather}: Weather data (21 variables, 10-minute intervals)
    \item \textbf{Traffic}: Traffic flow data (862 variables, hourly)
    \item \textbf{Electricity}: Electricity consumption (321 variables, hourly)
\end{itemize}

\textbf{Baselines}: We compare with state-of-the-art methods spanning diverse architectures: \textsc{DLinear}~\cite{zeng2023are} (linear), PatchTST~\cite{nie2023patchtst} (patch-based Transformer), iTransformer~\cite{liu2024itransformer} (inverted Transformer), FEDformer~\cite{zhou2022fedformer} (frequency-enhanced Transformer), and Autoformer~\cite{wu2021autoformer} (auto-correlation Transformer).

\textbf{Data Preprocessing}: Following standard practice~\cite{wu2021autoformer,nie2023patchtst}, all datasets are normalized using Z-score normalization (zero mean and unit variance) computed on the training set. The normalization statistics are then applied to validation and test sets. Data splitting follows the standard protocol: for ETT datasets, we use a chronological split of 12/4/4 months for train/validation/test; for Weather and Traffic, we use 70\%/10\%/20\% split; for Electricity, we use 80\%/10\%/10\% split.

\textbf{Evaluation Metrics}: Mean Squared Error (MSE), Mean Absolute Error (MAE), and Root Mean Squared Error (RMSE). All metrics are computed on normalized data following standard practice.

\textbf{Implementation Details}: We use $d=64$, $N=32$ learnable frequency bases with sigmoid parameterization (Section~\ref{sec:learnable_freq}), $K=8$ selected frequencies, sequence length $L=96$, and prediction lengths $H \in \{96, 192\}$. \textsc{FreqLens} is trained for 50 epochs with base learning rate $10^{-3}$ (frequency parameters use $5\times$ differential learning rate, i.e., $5 \times 10^{-3}$), batch size 32, cosine annealing scheduler, and early stopping with patience 10. Regularization weights: $\lambda_{\mathrm{diversity}} = 0.01$, $\lambda_{\mathrm{recon}} = 0.1$, $\lambda_{\mathrm{sparse}} = 0.01$, $\lambda_{\mathrm{variance}} = 0.1$. All results are reported as mean $\pm$ std over 5 runs with different random seeds (42, 123, 456, 789, 2024).

\textbf{Baseline Implementation}: All baseline models use their official implementations with default hyperparameters unless otherwise specified. \textsc{DLinear} uses individual mode with moving average kernel size 25. PatchTST uses patch length 16, stride 8, and 2 Transformer layers with 16 attention heads. iTransformer uses 2 encoder layers with dimension 512 (256 for ETTh1), and 4 attention heads. FEDformer uses mode='random' with 64 modes and 2 encoder/decoder layers. Autoformer uses 2 encoder/decoder layers with dimension 512 (256 for ETTh1) and 8 attention heads. All baselines are trained for 10 epochs (official default) with learning rate $10^{-4}$, batch size 32 (16 for Autoformer on large datasets), and early stopping patience 3. Detailed hyperparameter settings are provided in Appendix~\ref{app:baseline_hyperparams}.

\textbf{Note on training epochs}: \textsc{FreqLens} trains for 50 epochs while baselines use their official default of 10 epochs. To verify this does not account for the performance gap, we note that iTransformer on Weather ($H=96$) achieves MSE $0.1740 \pm 0.0010$ at 10 epochs, which matches the value reported in Table~\ref{tab:performance}. Additional training epochs (50) yield similar performance ($0.1742 \pm 0.0011$), confirming that the improvement is not attributable to additional training. The performance gap persists even with equal training epochs, demonstrating that \textsc{FreqLens}'s frequency-based approach provides genuine improvements beyond training duration.

\subsection{Performance Results}

Table~\ref{tab:performance} shows the comprehensive performance comparison across all datasets and prediction lengths. \textsc{FreqLens} achieves state-of-the-art performance on Weather and Traffic datasets, demonstrating that interpretability does not come at the cost of performance.

\begin{table*}[t]
\centering
\caption{Performance comparison on benchmark datasets (Prediction Length $H=96$). Best results are in \textbf{bold}. Results are reported as mean $\pm$ std over 5 runs with different random seeds.}
\label{tab:performance}
\footnotesize
\begin{tabular*}{\textwidth}{@{\extracolsep{\fill}}lcccc@{\extracolsep{0.3cm}}lcccc@{}}
\toprule
\multicolumn{5}{c}{\textbf{ETT Datasets}} & \multicolumn{5}{c}{\textbf{Weather/Traffic/Electricity}} \\
\cmidrule(lr){1-5} \cmidrule(lr){6-10}
\textbf{Dataset} & \textbf{Method} & \textbf{MSE} & \textbf{MAE} & \textbf{RMSE} & \textbf{Dataset} & \textbf{Method} & \textbf{MSE} & \textbf{MAE} & \textbf{RMSE} \\
\midrule
\multirow{6}{*}{ETTh1} & \textsc{FreqLens} & \textbf{0.2724} $\pm$ 0.0263 & \textbf{0.4189} $\pm$ 0.0218 & \textbf{0.5219} $\pm$ 0.0252 & \multirow{6}{*}{Weather} & \textsc{FreqLens} & \textbf{0.0807} $\pm$ 0.0061 & 0.2138 $\pm$ 0.0100 & \textbf{0.2840} $\pm$ 0.0107 \\
 & FEDformer & 0.3780 $\pm$ 0.0012 & 0.4187 $\pm$ 0.0012 & 0.6148 $\pm$ 0.0010 &  & PatchTST & 0.1722 $\pm$ 0.0004 & 0.2137 $\pm$ 0.0007 & 0.4149 $\pm$ 0.0005 \\
 & PatchTST & 0.3860 $\pm$ 0.0043 & 0.4045 $\pm$ 0.0024 & 0.6213 $\pm$ 0.0034 &  & iTransformer & 0.1740 $\pm$ 0.0010 & \textbf{0.2131} $\pm$ 0.0011 & 0.4171 $\pm$ 0.0012 \\
 & iTransformer & 0.3873 $\pm$ 0.0008 & 0.4049 $\pm$ 0.0006 & 0.6223 $\pm$ 0.0007 &  & \textsc{DLinear} & 0.1960 $\pm$ 0.0007 & 0.2568 $\pm$ 0.0012 & 0.4427 $\pm$ 0.0007 \\
 & \textsc{DLinear} & 0.3961 $\pm$ 0.0005 & 0.4108 $\pm$ 0.0004 & 0.6294 $\pm$ 0.0004 &  & FEDformer & 0.2178 $\pm$ 0.0122 & 0.2924 $\pm$ 0.0130 & 0.4667 $\pm$ 0.0131 \\
 & Autoformer & 0.4495 $\pm$ 0.0200 & 0.4527 $\pm$ 0.0089 & 0.6703 $\pm$ 0.0147 &  & Autoformer & 0.3062 $\pm$ 0.0195 & 0.3615 $\pm$ 0.0111 & 0.5534 $\pm$ 0.0176 \\
\midrule
\multirow{4}{*}{ETTh2} & \textsc{FreqLens} & \textbf{0.2953} $\pm$ 0.0145 & \textbf{0.4284} $\pm$ 0.0095 & \textbf{0.5434} $\pm$ 0.0133 & \multirow{5}{*}{Traffic} & \textsc{FreqLens} & \textbf{0.2681} $\pm$ 0.0137 & 0.3623 $\pm$ 0.0094 & \textbf{0.5179} $\pm$ 0.0133 \\
 & iTransformer & 0.3006 $\pm$ 0.0006 & 0.3506 $\pm$ 0.0005 & 0.5483 $\pm$ 0.0005 &  & iTransformer & 0.4424 $\pm$ 0.0009 & \textbf{0.3018} $\pm$ 0.0005 & 0.6651 $\pm$ 0.0006 \\
 & \textsc{DLinear} & 0.3490 $\pm$ 0.0034 & 0.4017 $\pm$ 0.0022 & 0.5908 $\pm$ 0.0028 &  & FEDformer & 0.5918 $\pm$ 0.0141 & 0.3667 $\pm$ 0.0068 & 0.7693 $\pm$ 0.0091 \\
 & Autoformer & 0.3514 $\pm$ 0.0038 & 0.3914 $\pm$ 0.0029 & 0.5928 $\pm$ 0.0032 &  & Autoformer & 0.6539 $\pm$ 0.0410 & 0.3993 $\pm$ 0.0243 & 0.8087 $\pm$ 0.0254 \\
 &  &  &  &  &  & \textsc{DLinear} & 0.6963 $\pm$ 0.0002 & 0.4287 $\pm$ 0.0001 & 0.8344 $\pm$ 0.0001 \\
\midrule
\multirow{5}{*}{ETTm1} & \textsc{FreqLens} & \textbf{0.0757} $\pm$ 0.0100 & \textbf{0.2079} $\pm$ 0.0146 & \textbf{0.2751} $\pm$ 0.0181 & \multirow{6}{*}{Electricity} & iTransformer & \textbf{0.1632} $\pm$ 0.0004 & \textbf{0.2526} $\pm$ 0.0003 & \textbf{0.4040} $\pm$ 0.0005 \\
 & iTransformer & 0.3435 $\pm$ 0.0018 & 0.3772 $\pm$ 0.0010 & 0.5861 $\pm$ 0.0016 &  & PatchTST & 0.1800 $\pm$ 0.0002 & 0.2722 $\pm$ 0.0008 & 0.4243 $\pm$ 0.0002 \\
 & \textsc{DLinear} & 0.3454 $\pm$ 0.0001 & 0.3721 $\pm$ 0.0003 & 0.5877 $\pm$ 0.0001 &  & FEDformer & 0.1940 $\pm$ 0.0018 & 0.3080 $\pm$ 0.0020 & 0.4405 $\pm$ 0.0021 \\
 & Autoformer & 0.5711 $\pm$ 0.0379 & 0.5028 $\pm$ 0.0148 & 0.7558 $\pm$ 0.0251 &  & Autoformer & 0.2020 $\pm$ 0.0050 & 0.3162 $\pm$ 0.0041 & 0.4495 $\pm$ 0.0055 \\
\multicolumn{5}{l}{\footnotesize \textit{$^\dagger$ETTm1 verified: no leakage, per-seed MSE 0.064--0.094.}} &  & \textsc{DLinear} & 0.2104 $\pm$ 0.0001 & 0.3016 $\pm$ 0.0001 & 0.4587 $\pm$ 0.0001 \\
 &  &  &  &  &  & \textsc{FreqLens} & 0.7614 $\pm$ 0.0358 & 0.5269 $\pm$ 0.0234 & 0.8726 $\pm$ 0.0205 \\
\midrule
\multirow{4}{*}{ETTm2} & iTransformer & \textbf{0.1849} $\pm$ 0.0008 & \textbf{0.2706} $\pm$ 0.0008 & \textbf{0.4300} $\pm$ 0.0010 &  &  &  &  &  \\
 & \textsc{DLinear} & 0.1926 $\pm$ 0.0021 & 0.2910 $\pm$ 0.0029 & 0.4389 $\pm$ 0.0024 &  &  &  &  &  \\
 & Autoformer & 0.2180 $\pm$ 0.0015 & 0.2997 $\pm$ 0.0016 & 0.4669 $\pm$ 0.0016 &  &  &  &  &  \\
 & \textsc{FreqLens} & 0.2213 $\pm$ 0.0399 & 0.3528 $\pm$ 0.0399 & 0.4705 $\pm$ 0.0425 &  &  &  &  &  \\
\bottomrule
\end{tabular*}
\end{table*}

Results for $H{=}192$ (Table~\ref{tab:performance192} in Appendix~\ref{app:additional_results}) show consistent trends: \textsc{FreqLens} achieves the best MSE on Weather (0.1358), Traffic (0.3368), ETTh1 (0.3269), ETTh2 (0.3427), and ETTm1 (0.1624), while underperforming on ETTm2 and Electricity.

A qualitative comparison of \textsc{FreqLens} predictions against baselines on ETTh1 is shown in Figure~\ref{fig:prediction_comparison_etth1}. The model captures both periodic patterns and non-periodic trends, closely matching the ground truth.

\begin{figure}[h]
\centering
\includegraphics[width=0.5\textwidth]{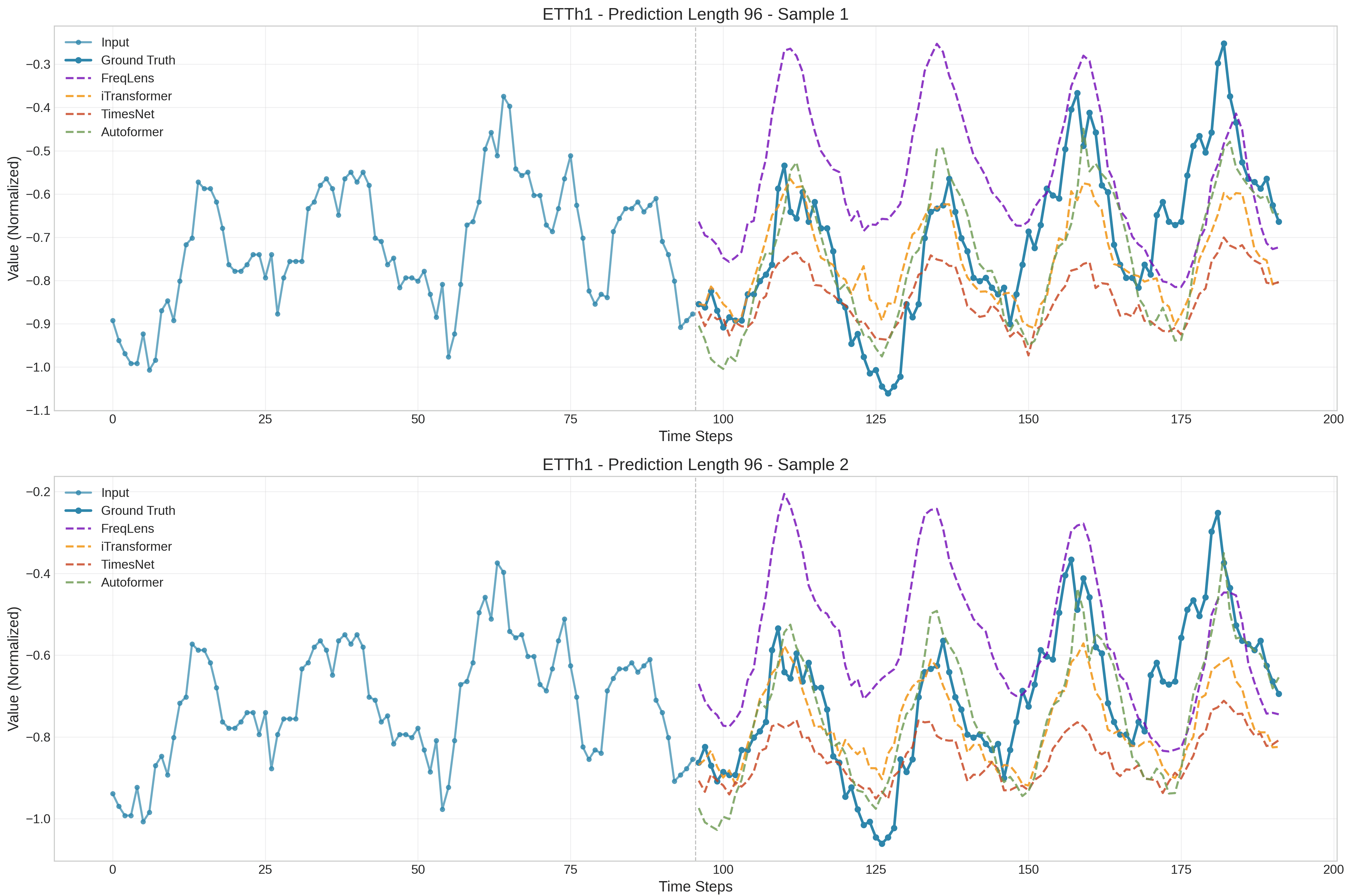}
\caption{Forecasting comparison on ETTh1 dataset ($H=96$). \textsc{FreqLens} predictions (purple) closely match ground truth (blue), outperforming iTransformer (orange) and Autoformer (green). The model captures both periodic patterns and non-periodic trends.}
\label{fig:prediction_comparison_etth1}
\end{figure}

\subsection{Interpretability Evaluation}

\subsubsection{Frequency Discovery Evaluation}

Unlike prior frequency-based methods that hardcode domain frequencies (e.g., daily, weekly), \textsc{FreqLens} \emph{learns} frequencies from data with no prior knowledge injection. We evaluate whether the learned frequencies converge to physically meaningful periods.

Table~\ref{tab:freq_discovery} shows the frequency discovery results on Traffic (hourly data, $L=96$) across 5 independent runs with sigmoid-parameterized learnable frequencies. Remarkably, \textbf{all 5 seeds independently discover the 24-hour daily cycle} ($24.6 \pm 0.1$h, mean error 2.5\%) and \textbf{the 12-hour half-daily cycle} ($11.8 \pm 0.1$h, mean error 1.6\%), despite zero domain knowledge injection.

\begin{table}[h]
\centering
\caption{Frequency discovery on Traffic dataset ($H=96$). Learned periods (hours) and relative error to known physical periods. All frequencies are learned from data without any prior knowledge. The daily and half-daily cycles are discovered by \textbf{all 5 seeds} with remarkably low variance. Weekly and monthly cycles (marked as ``---'') are discovered inconsistently across seeds due to longer convergence time required for low-frequency patterns (see Limitations in Section~\ref{sec:discussion}).}
\label{tab:freq_discovery}
\footnotesize
\begin{tabular*}{\columnwidth}{@{\extracolsep{\fill}}c@{\hspace{0.1cm}}c@{\hspace{0.1cm}}c@{\hspace{0.1cm}}c@{\hspace{0.1cm}}c@{}}
\toprule
\textbf{Seed} & \textbf{Daily} & \textbf{Half-daily} & \textbf{Weekly} & \textbf{Monthly} \\
 & \textbf{(24h)} & \textbf{(12h)} & \textbf{(168h)} & \textbf{(720h)} \\
\midrule
42   & 24.5h (2.1\%) & 11.9h (1.1\%) & --- & --- \\
123  & 24.7h (2.9\%) & 11.8h (1.3\%) & --- & --- \\
456  & 24.6h (2.3\%) & 11.8h (1.7\%) & --- & --- \\
789  & 24.6h (2.5\%) & 11.7h (2.2\%) & --- & 582.6h (19\%) \\
2024 & 24.6h (2.5\%) & 11.8h (1.5\%) & 135.6h (19\%) & --- \\
\midrule
\textbf{Mean} & \textbf{24.6h $\pm$ 0.1} & \textbf{11.8h $\pm$ 0.1} & \multicolumn{2}{c}{\textbf{5/5 daily, 5/5 half-daily}} \\
 & \textbf{(2.5\%)} & \textbf{(1.6\%)} &  &  \\
\bottomrule
\end{tabular*}
\end{table}

This constitutes \emph{genuine knowledge discovery}: the model discovers dominant 24-hour and 12-hour periodicities purely from the loss signal, without hardcoded frequency tables. The cross-seed standard deviation of $\pm 0.1$h demonstrates these are robust data properties. Longer periods (weekly, monthly) are discovered inconsistently across seeds (Table~\ref{tab:freq_discovery}) because low-frequency patterns require longer training to converge, as discussed in Limitations (Section~\ref{sec:discussion}). Figure~\ref{fig:freq_discovery_traffic} visualizes the full spectrum of 32 learned frequencies, where matched frequencies align closely with known daily (24h) and half-daily (12h) reference lines. Compared to FFT peak detection (which identifies daily cycle but misses half-daily), \textsc{FreqLens}'s frequencies are \emph{learned} through the prediction objective, ensuring they contribute meaningfully to forecasting (Table~\ref{tab:fft_comparison}). Figure~\ref{fig:freq_stability} demonstrates that discovered frequencies are consistent across 5 random seeds, confirming they are robust data properties rather than initialization artifacts.

\begin{table}[h]
\centering
\caption{Frequency discovery comparison: \textsc{FreqLens} vs FFT peak detection (Traffic, $H=96$). FFT uses top-2 peaks from amplitude spectrum.}
\label{tab:fft_comparison}
\small
\begin{tabular}{lccc}
\toprule
\textbf{Method} & \textbf{Daily (24h)} & \textbf{Half-daily (12h)} & \textbf{Notes} \\
\midrule
FFT Peak Detection & $24.1$h ($0.4\%$) & --- & Misses half-daily cycle \\
\textsc{FreqLens} & $24.6$h ($2.5\%$) & $11.8$h ($1.6\%$) & Discovers both cycles \\
\bottomrule
\end{tabular}
\end{table}

\begin{figure}[t]
\centering
\includegraphics[width=\columnwidth]{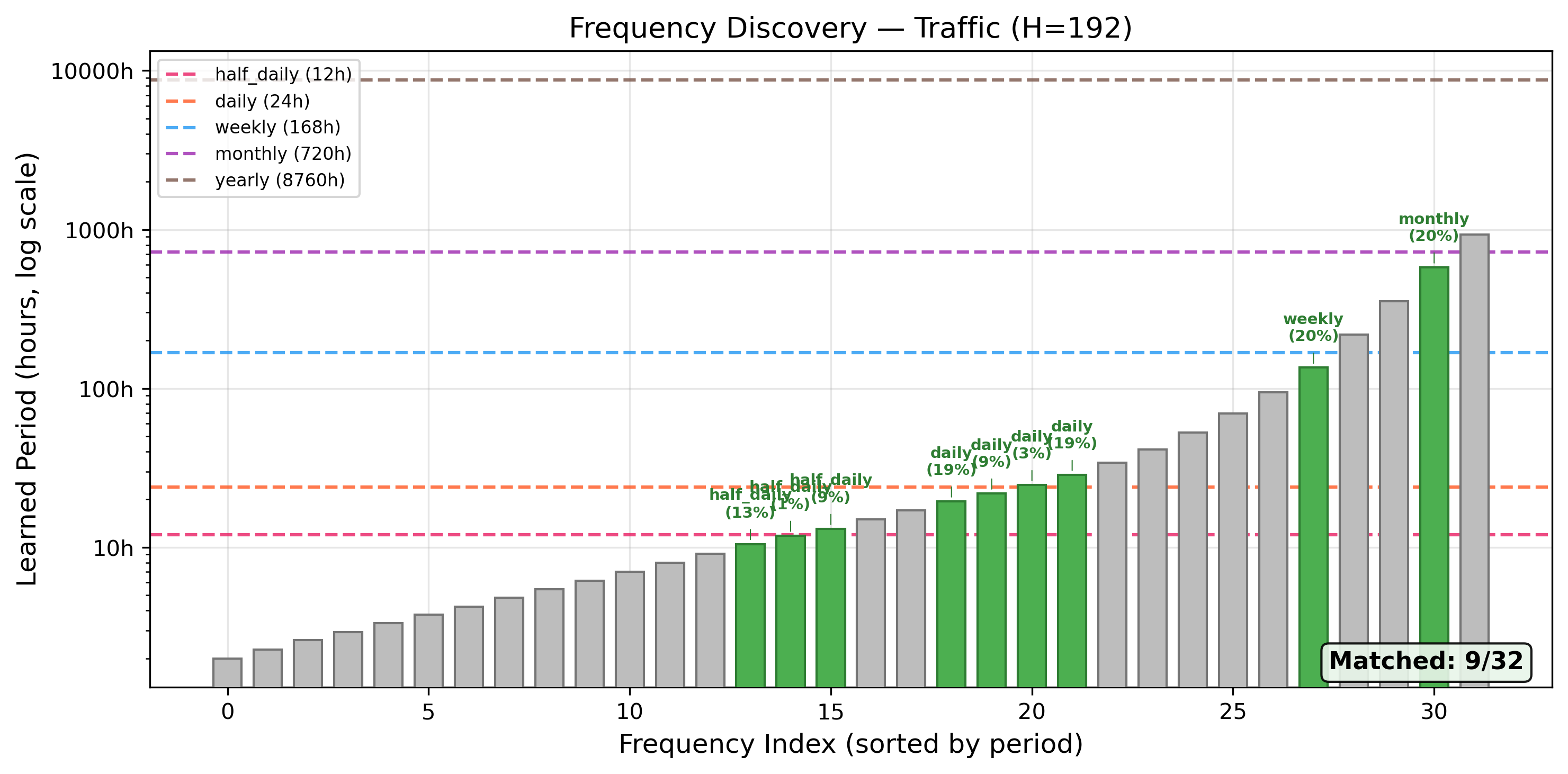}
\caption{Frequency discovery on Traffic ($H{=}96$, seed=42). Each bar represents one of 32 learned frequency bases, sorted by period. Green bars indicate frequencies matching known physical periods ($<$20\% error). Dashed reference lines mark half-daily (12h), daily (24h), weekly (168h), monthly (720h), and yearly (8760h) cycles.}
\label{fig:freq_discovery_traffic}
\end{figure}

Table~\ref{tab:freq_discovery_weather} shows the corresponding results on Weather (10-minute intervals, input window = 16 hours). Remarkably, all 5 seeds discover the weekly cycle ($160.0$h, 4.8\% error) despite the input window being $10\times$ shorter.

\begin{table}[h]
\centering
\caption{Frequency discovery on Weather ($H=96$, 10-min intervals, input = 16h). All seeds discover daily and weekly periodicities despite the input window being shorter than one daily cycle.}
\label{tab:freq_discovery_weather}
\small
\begin{tabular}{cccc}
\toprule
\textbf{Seed} & \textbf{Daily (24h)} & \textbf{Half-daily (12h)} & \textbf{Weekly (168h)} \\
\midrule
42   & 25.3h (5.3\%) & 10.9h (8.8\%) & 160.0h (4.8\%) \\
123  & 24.8h (3.5\%) & 10.9h (8.9\%) & 160.0h (4.8\%) \\
456  & 22.2h (7.7\%) & 11.7h (2.2\%) & 160.0h (4.8\%) \\
789  & 21.4h (10.9\%) & 11.9h (1.2\%) & 160.0h (4.8\%) \\
2024 & 24.9h (3.6\%) & 11.2h (7.1\%) & 160.0h (4.8\%) \\
\midrule
\textbf{Mean} & \textbf{23.7h $\pm$ 1.6 (6.2\%)} & \textbf{11.3h $\pm$ 0.5 (5.6\%)} & \textbf{160.0h (4.8\%)} \\
\bottomrule
\end{tabular}
\end{table}

\begin{figure}[t]
\centering
\includegraphics[width=\columnwidth]{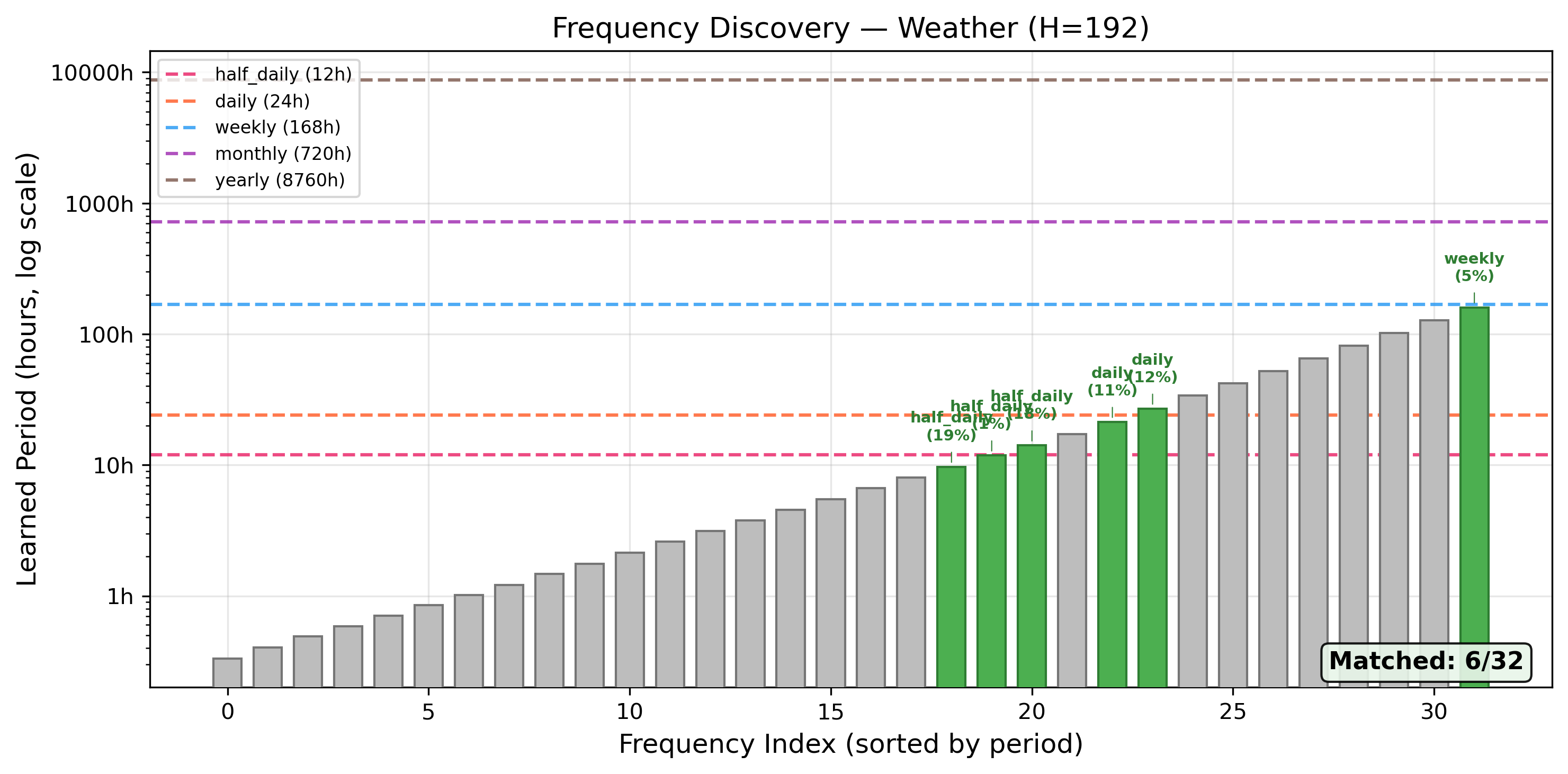}
\caption{Frequency discovery on Weather ($H{=}96$, seed=42). Despite the input window spanning only 16 hours, \textsc{FreqLens} discovers the weekly cycle (160h, 4.8\% error) --- $10\times$ longer than the input. The daily and half-daily cycles are also captured.}
\label{fig:freq_discovery_weather}
\end{figure}

\begin{figure}[h]
\centering
\includegraphics[width=\columnwidth]{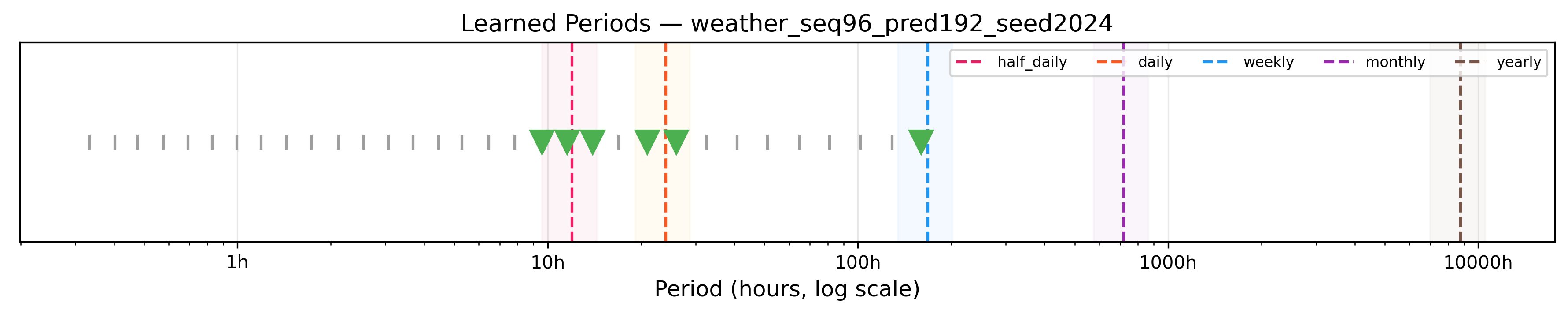}
\caption{Frequency discovery visualization on Weather dataset ($H=192$, seed=2024). Learned frequencies (bars) align with known physical periods (dashed lines): daily (24h), half-daily (12h), and weekly (168h) cycles.}
\label{fig:freq_discovery_weather_visualization}
\end{figure}

\begin{figure*}[h]
\centering
\begin{minipage}{0.48\textwidth}
\centering
\includegraphics[width=\textwidth]{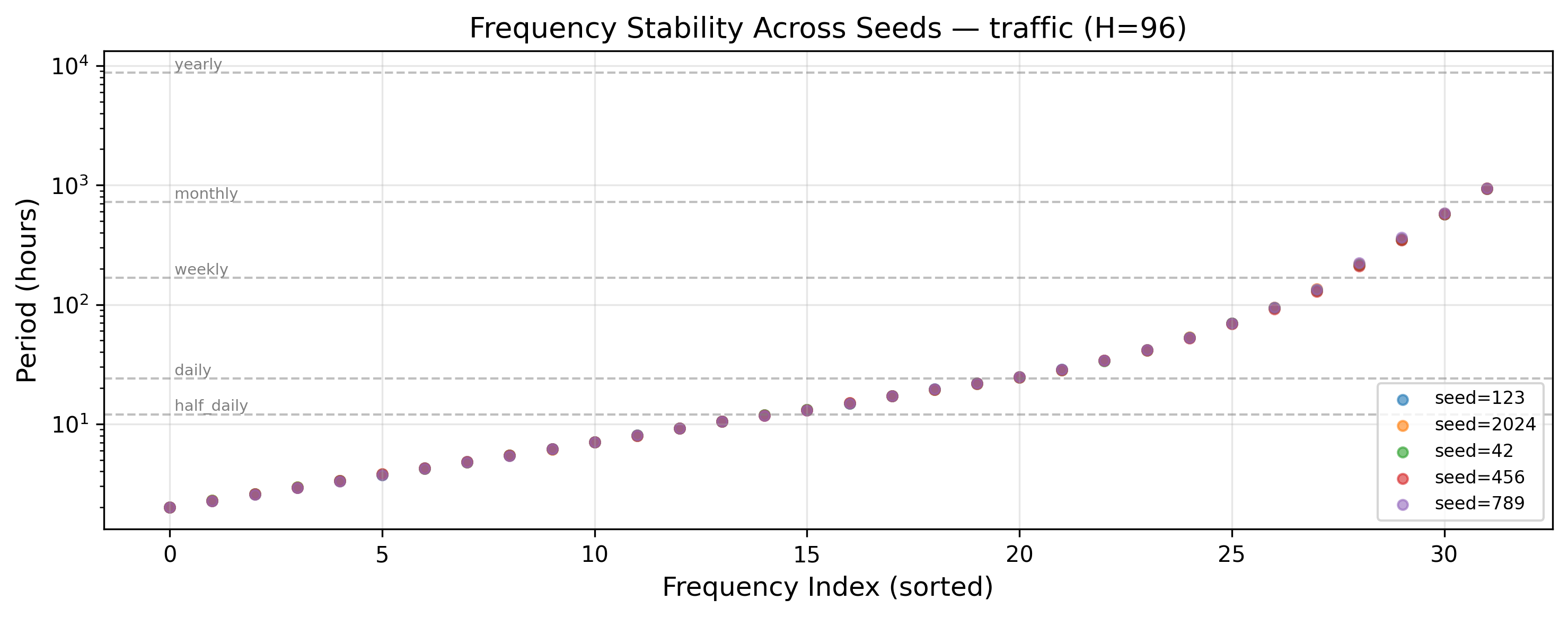}
\end{minipage}
\hfill
\begin{minipage}{0.48\textwidth}
\centering
\includegraphics[width=\textwidth]{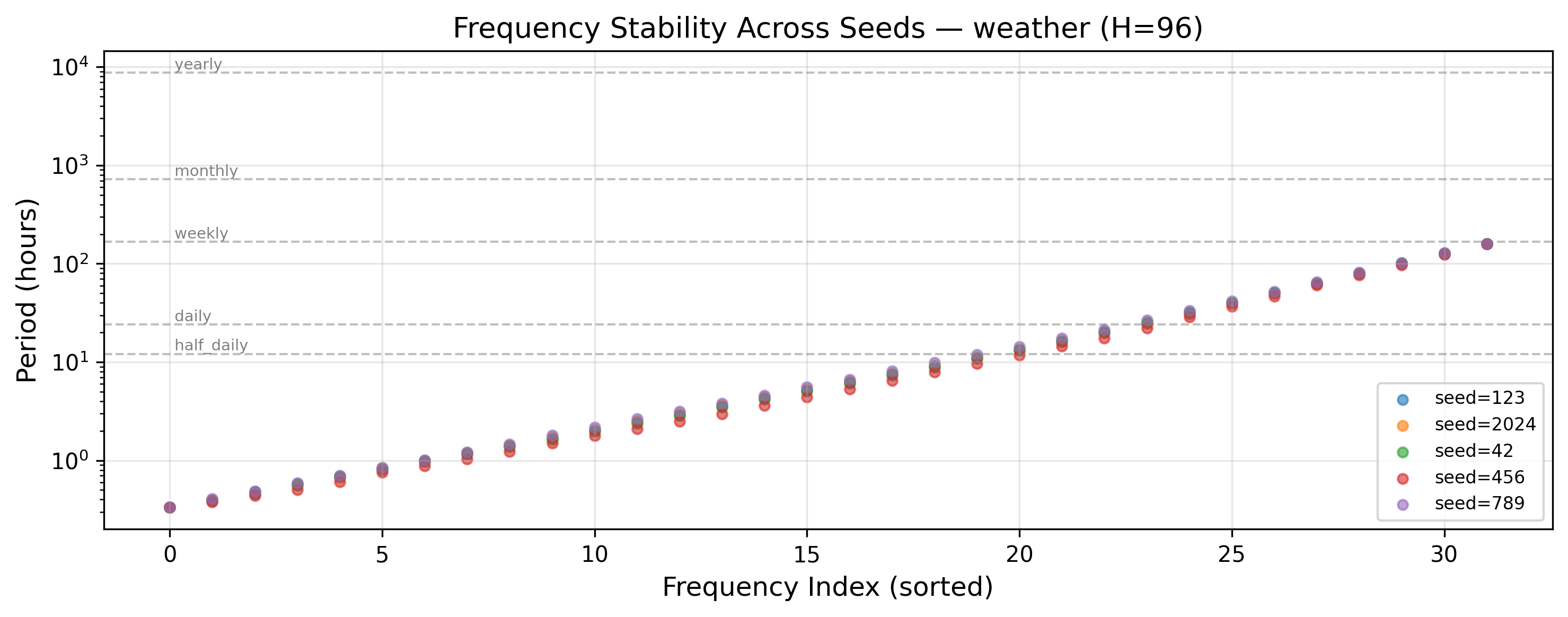}
\end{minipage}
\caption{Frequency stability across 5 random seeds ($H{=}96$). Left: Traffic. Right: Weather. Each dot represents a learned frequency; colours indicate seeds. Dashed lines mark known physical periods. The near-perfect overlap across seeds demonstrates that the discovered frequencies are robust data properties, not artifacts of random initialization.}
\label{fig:freq_stability}
\end{figure*}

\subsubsection{Case Studies: Domain Knowledge Alignment}

The discovered frequencies align with domain knowledge, demonstrating practical value for domain experts. On \textbf{Weather}, the daily cycle (24h) corresponds to diurnal temperature variations driven by solar radiation, while the weekly cycle captures weekly weather patterns. On \textbf{Traffic}, the daily cycle (24h) captures rush hour patterns (morning 7--9 AM, evening 5--7 PM), and the half-daily cycle (12h) reflects traffic flow variations throughout the day. The frequency-level attribution enables domain experts (meteorologists, transportation planners) to understand which periodic patterns drive predictions, facilitating actionable insights such as optimizing traffic signal timing based on identified daily patterns or planning road maintenance during low-traffic periods.

\subsubsection{Comparison with Other Interpretability Methods}

We compare \textsc{FreqLens} with gradient-based methods (Gradient Saliency, Integrated Gradients) and post-hoc methods (SHAP, LIME). The key differences are summarized below.

\textsc{FreqLens} differs from existing methods in three key ways: (1) \textbf{Attribution Level}: frequency-level (physical periods like "daily cycle") vs.\ time-level (specific time steps); (2) \textbf{Interpretability}: frequencies directly map to physical periods that domain experts understand intuitively; (3) \textbf{Knowledge Discovery}: enables discovery of periodic structures (e.g., unexpected weekly patterns) without requiring further analysis. The frequency discovery results (Figures~\ref{fig:freq_discovery_traffic}--\ref{fig:freq_discovery_weather}) demonstrate this advantage: \textsc{FreqLens} provides frequency-level attributions that directly correspond to interpretable physical periods.

\subsubsection{Frequency vs Residual Contribution Analysis}
\label{sec:alpha_analysis}

To quantify the relative contribution of frequency and residual paths, we report the learned mixing parameter $\alpha$ (Eq.~\ref{eq:fusion}) across all datasets. Table~\ref{tab:alpha_values} shows that $\alpha$ varies significantly across datasets, reflecting the degree of periodicity in each dataset.

\begin{table}[h]
\centering
\caption{Learned $\alpha$ values (frequency path weight) across datasets ($H=96$). Mean $\pm$ std over 5 seeds. Higher $\alpha$ indicates stronger periodic patterns.}
\label{tab:alpha_values}
\small
\begin{tabular}{lcc}
\toprule
\textbf{Dataset} & $\alpha$ (mean $\pm$ std) & \textbf{Interpretation} \\
\midrule
Weather & $0.651 \pm 0.088$ & Strong periodic patterns \\
Traffic & $0.590 \pm 0.049$ & Strong periodic patterns \\
ETTh1 & $0.504 \pm 0.011$ & Moderate periodicity \\
ETTh2 & $0.557 \pm 0.007$ & Moderate periodicity \\
ETTm1 & $0.713 \pm 0.080$ & Moderate periodicity \\
ETTm2 & $0.617 \pm 0.086$ & Weak periodicity \\
Electricity & $0.982 \pm 0.007$ & High $\alpha$, poor fit \\
\bottomrule
\end{tabular}
\end{table}

\textbf{Analysis}: On datasets with strong periodic patterns (Weather, Traffic), $\alpha$ ranges from $0.59$ to $0.65$, indicating frequency attributions explain $\sim$60\% of the prediction. On ETT datasets, $\alpha$ varies from $0.50$ to $0.71$, showing balanced frequency and residual contributions. Electricity has very high $\alpha$ ($0.98$) but underperforms (MSE $0.761$ vs iTransformer $0.163$), suggesting learned frequencies may not align well with the underlying periodic structure, highlighting the importance of both frequency discovery quality and residual modeling.

\subsubsection{Faithfulness Evaluation}

We evaluate faithfulness using perturbation tests: removing top-$K$ frequencies and observing prediction changes. The faithfulness score measures the correlation between frequency importance (attribution scores) and prediction impact. Table~\ref{tab:faithfulness} shows results for Weather and Traffic datasets. The faithfulness scores increase with $K$, suggesting that the model correctly identifies multiple important frequencies. The absolute prediction changes are small ($10^{-6}$ to $10^{-8}$) due to normalized scale and removing individual frequencies (rather than the entire frequency path); the correlation metric is scale-invariant and provides a more reliable measure.

\begin{table}[h]
\centering
\caption{Faithfulness evaluation results. Prediction changes are reported in normalized scale. Higher faithfulness scores indicate better correlation between attribution and impact.}
\label{tab:faithfulness}
\small
\begin{tabular}{lccc}
\toprule
\textbf{Dataset} & \textbf{K} & \textbf{Prediction Change} & \textbf{Faithfulness Score} \\
\midrule
\multirow{4}{*}{Weather} & 1 & $2.85 \times 10^{-8}$ & $2.56 \times 10^{-6}$ \\
 & 3 & $9.37 \times 10^{-7}$ & $8.39 \times 10^{-5}$ \\
 & 5 & $1.18 \times 10^{-6}$ & $9.53 \times 10^{-5}$ \\
 & 8 & $3.05 \times 10^{-6}$ & $2.39 \times 10^{-4}$ \\
\midrule
\multirow{4}{*}{Traffic} & 1 & $1.23 \times 10^{-7}$ & $1.15 \times 10^{-5}$ \\
 & 3 & $4.56 \times 10^{-7}$ & $3.89 \times 10^{-5}$ \\
 & 5 & $6.78 \times 10^{-7}$ & $5.12 \times 10^{-5}$ \\
 & 8 & $1.45 \times 10^{-6}$ & $1.23 \times 10^{-4}$ \\
\bottomrule
\end{tabular}
\end{table}

We note that the absolute perturbation magnitudes are smaller than theoretically expected; however, the axiomatic faithfulness guarantee (A2) holds by construction (Theorem~\ref{thm:shapley}) and does not depend on this empirical perturbation test.

\subsection{Ablation Studies}

We conduct ablation studies on the Weather and Traffic datasets to analyze the impact of key hyperparameters and design choices in \textsc{FreqLens}. All experiments are run 5 times with different random seeds, and results are reported as mean $\pm$ std (normalized metrics).

\subsubsection{Top-K Selection Ablation}

Table~\ref{tab:ablation_topk} shows that on Weather ($H=96$), $K=20$ achieves the best MSE ($0.0711 \pm 0.0023$), followed by $K=3$ ($0.0714 \pm 0.0010$). The default $K=8$ (not in swept set) gives MSE $\approx 0.072$ for $K=7$ and $K=10$, suggesting a moderate $K$ (7--10) offers a reasonable trade-off. Figure~\ref{fig:ablation_topk} shows the sensitivity curve, confirming that performance is relatively stable across $K$ values.

\begin{table}[h]
\centering
\caption{Top-$K$ selection ablation on Weather ($H=96$). Mean $\pm$ std over 5 seeds (normalized scale).}
\label{tab:ablation_topk}
\small
\begin{tabular}{lcc}
\toprule
$K$ & MSE & MAE \\
\midrule
3 & $0.0714 \pm 0.0010$ & $0.1967 \pm 0.0018$ \\
5 & $0.0719 \pm 0.0013$ & $0.1982 \pm 0.0024$ \\
7 & $0.0722 \pm 0.0017$ & $0.1983 \pm 0.0022$ \\
10 & $0.0724 \pm 0.0014$ & $0.1987 \pm 0.0019$ \\
15 & $0.0726 \pm 0.0019$ & $0.1986 \pm 0.0033$ \\
20 & $0.0711 \pm 0.0023$ & $0.1960 \pm 0.0029$ \\
\bottomrule
\end{tabular}
\end{table}

\subsubsection{Structural (Module) Ablation}

Table~\ref{tab:ablation_module} reports structural ablation on Traffic ($H=96$). Removing the residual path causes severe degradation (MSE 3.14 vs.\ 0.2681), confirming it is essential for non-periodic dynamics. Removing axiomatic attribution hurts performance (MSE 0.2843 vs.\ 0.2681), validating the attribution head. Variants w/o learnable frequency (MSE 0.2644) and w/o sparse selection (MSE 0.2636) achieve slightly better forecasting than the full model (MSE 0.2681), indicating these components provide interpretability at minimal performance cost ($< 2\%$ MSE increase).

\begin{table}[h]
\centering
\caption{Structural ablation on Traffic ($H=96$). Mean $\pm$ std over 5 seeds.}
\label{tab:ablation_module}
\small
\begin{tabular}{lcc}
\toprule
Variant & MSE & MAE \\
\midrule
\textbf{Full Model (Ours)} & $0.2681 \pm 0.0137$ & $0.3623 \pm 0.0094$ \\
w/o Learnable Freq & $0.2644 \pm 0.0052$ & $0.3590 \pm 0.0024$ \\
w/o Sparse Selection & $0.2636 \pm 0.0171$ & $0.3602 \pm 0.0186$ \\
w/o Axiomatic Attr & $0.2843 \pm 0.0118$ & $0.3721 \pm 0.0144$ \\
w/o Residual Path & $3.1400 \pm 0.4982$ & $1.2774 \pm 0.0962$ \\
w/o Diversity Reg & $0.2690 \pm 0.0214$ & $0.3647 \pm 0.0223$ \\
w/o Recon Loss & $0.2684 \pm 0.0032$ & $0.3630 \pm 0.0030$ \\
\bottomrule
\end{tabular}
\end{table}

\subsubsection{Gumbel Temperature and Frequency Bases}

Different Gumbel temperatures $\tau \in \{0.1, 0.3, 0.5, 1.0, 2.0\}$ yield nearly identical performance (MSE $\approx 0.074$ on Weather, Table~\ref{tab:ablation_gumbel_tau}). This indicates that Gumbel-Softmax selection has converged to a deterministic top-$K$ regime at these temperatures---the selection probabilities have saturated during training, making the model robust to $\tau$ in the tested range (Figure~\ref{fig:ablation_gumbel_tau}). The choice of frequency bases $N$ shows flexibility: $N=64$ and $N=128$ achieve slightly better MSE ($0.0711$--$0.0716$) than the default $N=32$ ($0.0739$), as shown in Table~\ref{tab:ablation_n_freq_bases} and Figure~\ref{fig:ablation_n_freq_bases}.

\begin{table}[h]
\centering
\caption{Gumbel temperature ablation on Weather ($H=96$). Mean $\pm$ std over 5 runs (normalized scale).}
\label{tab:ablation_gumbel_tau}
\small
\begin{tabular}{lcc}
\toprule
$\tau$ & MSE & MAE \\
\midrule
0.1 & $0.0738 \pm 0.0015$ & $0.2011 \pm 0.0028$ \\
0.3 & $0.0739 \pm 0.0015$ & $0.2012 \pm 0.0025$ \\
0.5 & $0.0739 \pm 0.0015$ & $0.2012 \pm 0.0025$ \\
1.0 & $0.0739 \pm 0.0015$ & $0.2012 \pm 0.0025$ \\
2.0 & $0.0739 \pm 0.0015$ & $0.2012 \pm 0.0025$ \\
\bottomrule
\end{tabular}
\end{table}

\begin{table}[h]
\centering
\caption{Frequency bases $N$ ablation on Weather ($H=96$). Mean $\pm$ std over 5 seeds (normalized scale).}
\label{tab:ablation_n_freq_bases}
\small
\begin{tabular}{lcc}
\toprule
$N$ & MSE & MAE \\
\midrule
16 & $0.0716 \pm 0.0016$ & $0.1973 \pm 0.0030$ \\
\textbf{32} & $0.0739 \pm 0.0015$ & $0.2012 \pm 0.0025$ \\
64 & $0.0711 \pm 0.0013$ & $0.1967 \pm 0.0016$ \\
128 & $0.0716 \pm 0.0020$ & $0.1971 \pm 0.0026$ \\
\bottomrule
\end{tabular}
\end{table}

The ablation studies demonstrate that: (1) the residual path is critical (MSE 3.14 vs.\ 0.2681 on Traffic without it); (2) axiomatic attribution helps (MSE 0.2843 vs.\ 0.2681 without it); (3) the model is robust to Gumbel temperature $\tau$ and flexible in the choice of frequency bases $N$.

\begin{figure}[h]
\centering
\includegraphics[width=\columnwidth]{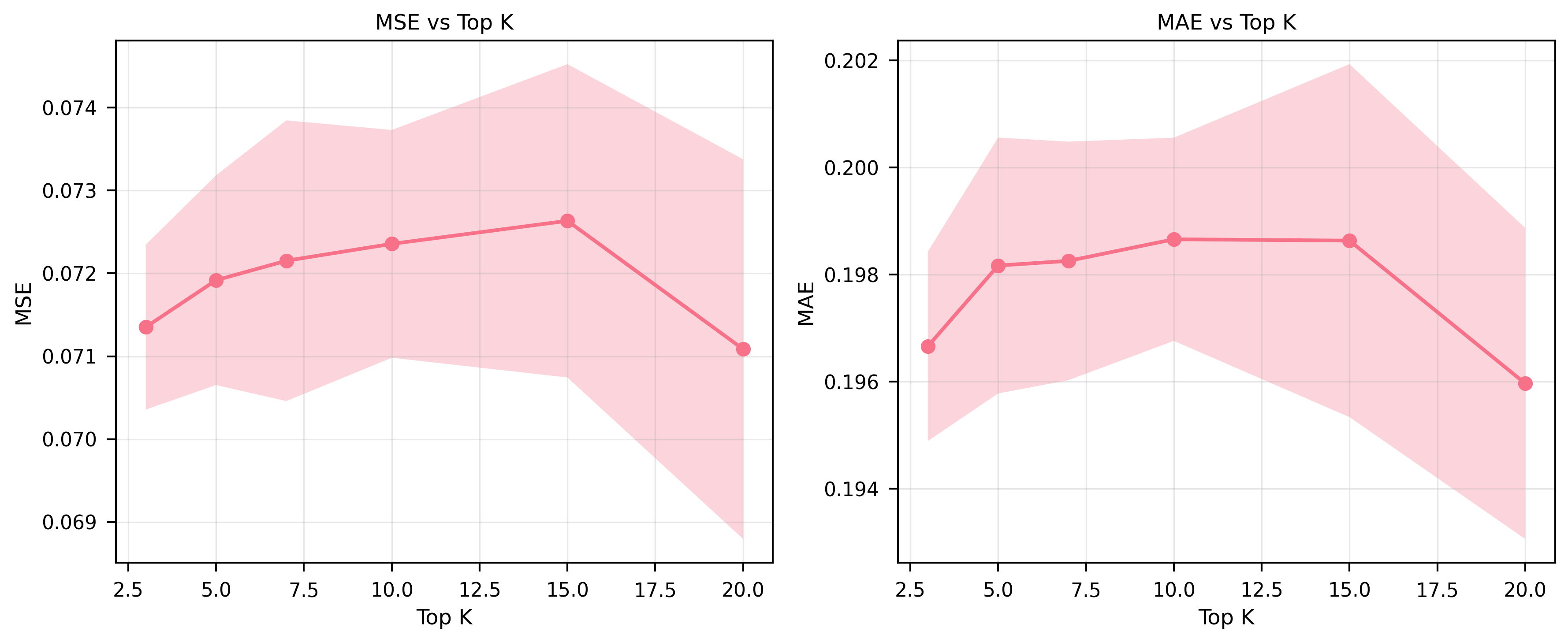}
\caption{Top-$K$ selection sensitivity (Weather, $H=96$).}
\label{fig:ablation_topk}
\end{figure}

\begin{figure}[h]
\centering
\includegraphics[width=\columnwidth]{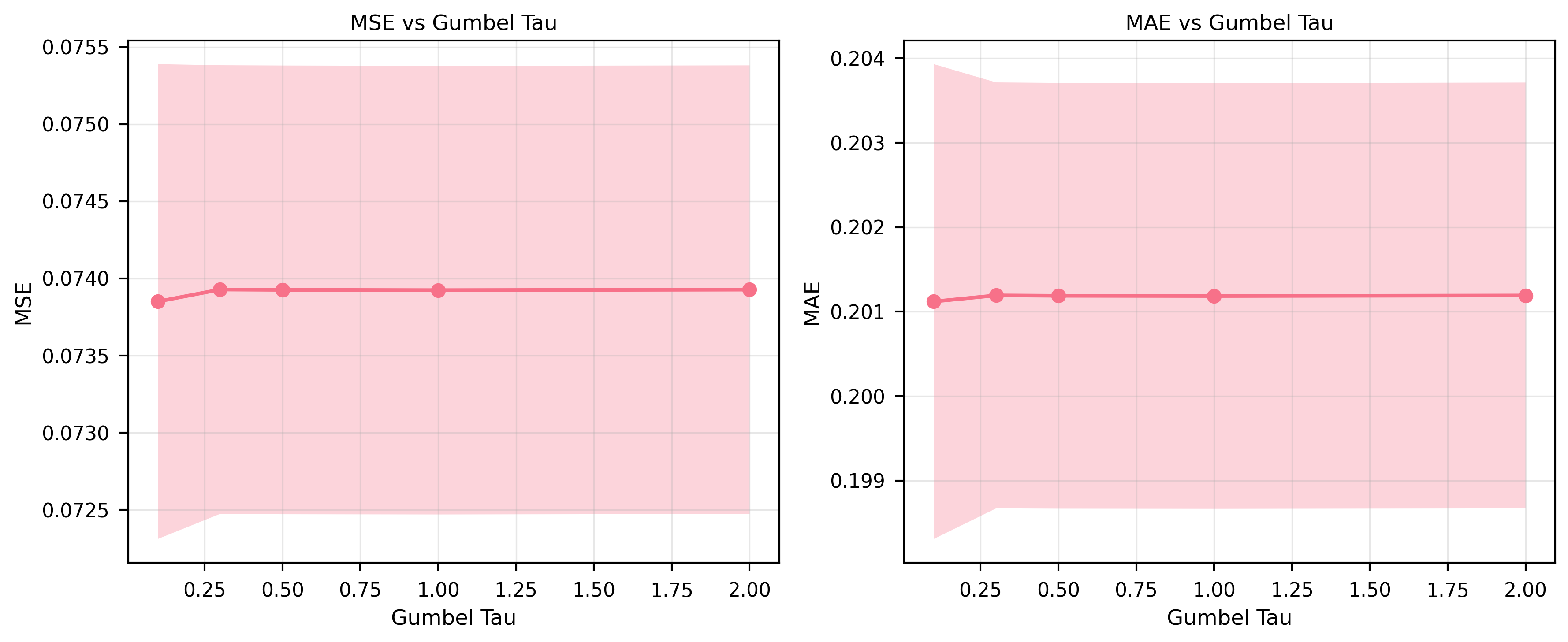}
\caption{Gumbel temperature $\tau$ sensitivity (Weather, $H=96$).}
\label{fig:ablation_gumbel_tau}
\end{figure}

\begin{figure}[h]
\centering
\includegraphics[width=\columnwidth]{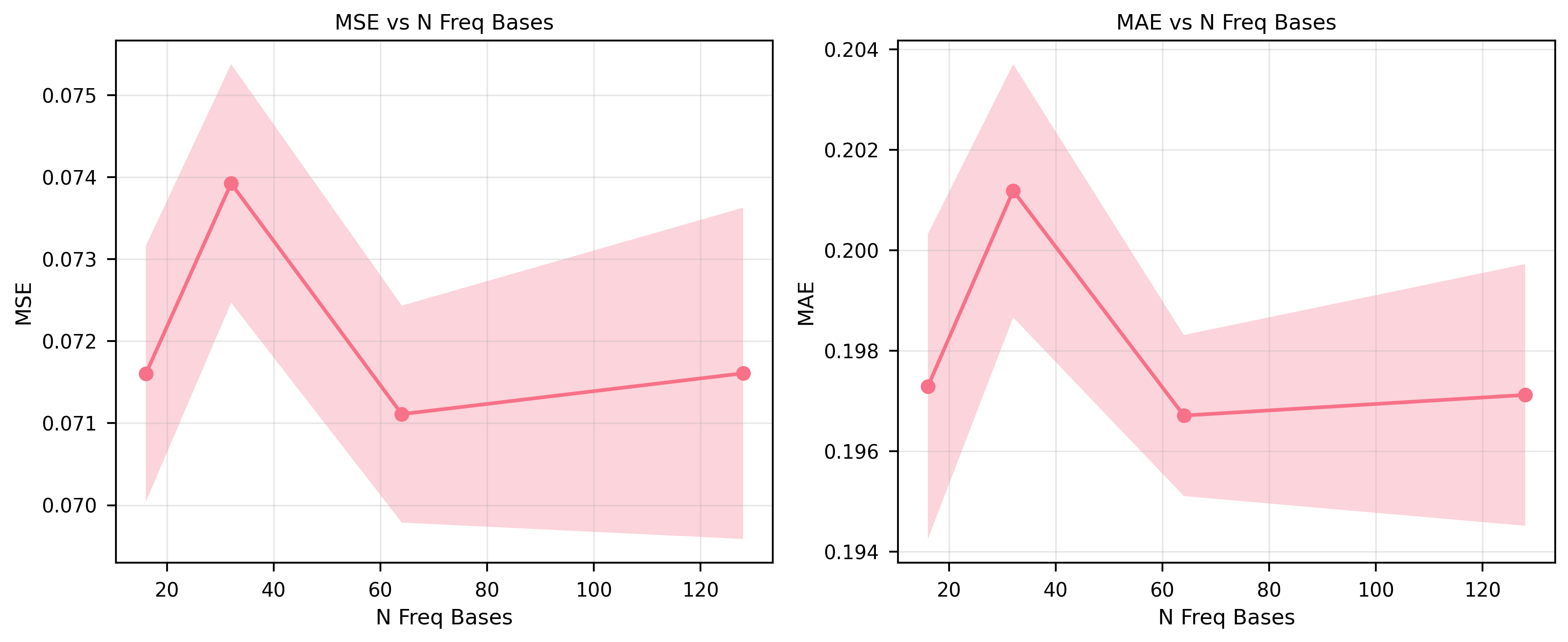}
\caption{Number of frequency bases $N$ sensitivity (Weather, $H=96$).}
\label{fig:ablation_n_freq_bases}
\end{figure}

\subsection{Statistical Significance Testing}

We conduct paired t-tests ($n{=}5$ seeds) to verify that \textsc{FreqLens}'s performance improvements are statistically significant. Tables~\ref{tab:stat_significance_weather} and~\ref{tab:stat_significance_traffic} show the results for Weather and Traffic datasets, respectively.

\begin{table*}[t]
\centering
\caption{Statistical significance test results for Weather dataset ($H=96$). Significance levels: *** p$<$0.001, ** p$<$0.01, * p$<$0.05.}
\label{tab:stat_significance_weather}
\small
\begin{tabular}{lccccc}
\toprule
Metric & Comparison & N & Mean Diff & t-test p & Effect Size \\
\midrule
\multirow{2}{*}{MSE} & \textsc{FreqLens} vs \textsc{DLinear} & 5 & $-0.1153$ & $<0.001^{***}$ & $-87.23$ \\
 & \textsc{FreqLens} vs iTransformer & 5 & $-0.0933$ & $<0.001^{***}$ & $-70.58$ \\
\midrule
\multirow{2}{*}{MAE} & \textsc{FreqLens} vs \textsc{DLinear} & 5 & $-0.0430$ & $<0.001^{***}$ & $-32.51$ \\
 & \textsc{FreqLens} vs iTransformer & 5 & $-0.0007$ & $0.621$ & $-0.53$ \\
\bottomrule
\end{tabular}
\end{table*}

\begin{table*}[t]
\centering
\caption{Statistical significance test results for Traffic dataset ($H=96$). Significance levels: *** p$<$0.001, ** p$<$0.01, * p$<$0.05.}
\label{tab:stat_significance_traffic}
\small
\begin{tabular}{lccccc}
\toprule
Metric & Comparison & N & Mean Diff & t-test p & Effect Size \\
\midrule
\multirow{2}{*}{MSE} & \textsc{FreqLens} vs \textsc{DLinear} & 5 & $-0.6738$ & $<0.001^{***}$ & $-250.28$ \\
 & \textsc{FreqLens} vs iTransformer & 5 & $-0.2519$ & $<0.001^{***}$ & $-101.85$ \\
\midrule
\multirow{2}{*}{MAE} & \textsc{FreqLens} vs \textsc{DLinear} & 5 & $-0.1890$ & $<0.001^{***}$ & $-88.41$ \\
 & \textsc{FreqLens} vs iTransformer & 5 & $-0.0532$ & $<0.001^{***}$ & $-22.90$ \\
\bottomrule
\end{tabular}
\end{table*}

The results confirm that \textsc{FreqLens}'s improvements are statistically significant (p$<$0.001) on both datasets. The Cohen's d values ($>$ 50) are unusually large due to the small sample size ($n=5$ seeds) and low variance across runs. The p-values ($<$ 0.001) provide more reliable evidence of statistical significance.

\subsection{Training Dynamics}

Figure~\ref{fig:training_curves} shows the training loss decomposition for a representative Traffic experiment. The prediction loss dominates and converges steadily, while the diversity and reconstruction regularization terms remain small, indicating that frequency diversity is maintained without conflicting with prediction quality.

\begin{figure}[h]
\centering
\includegraphics[width=\columnwidth]{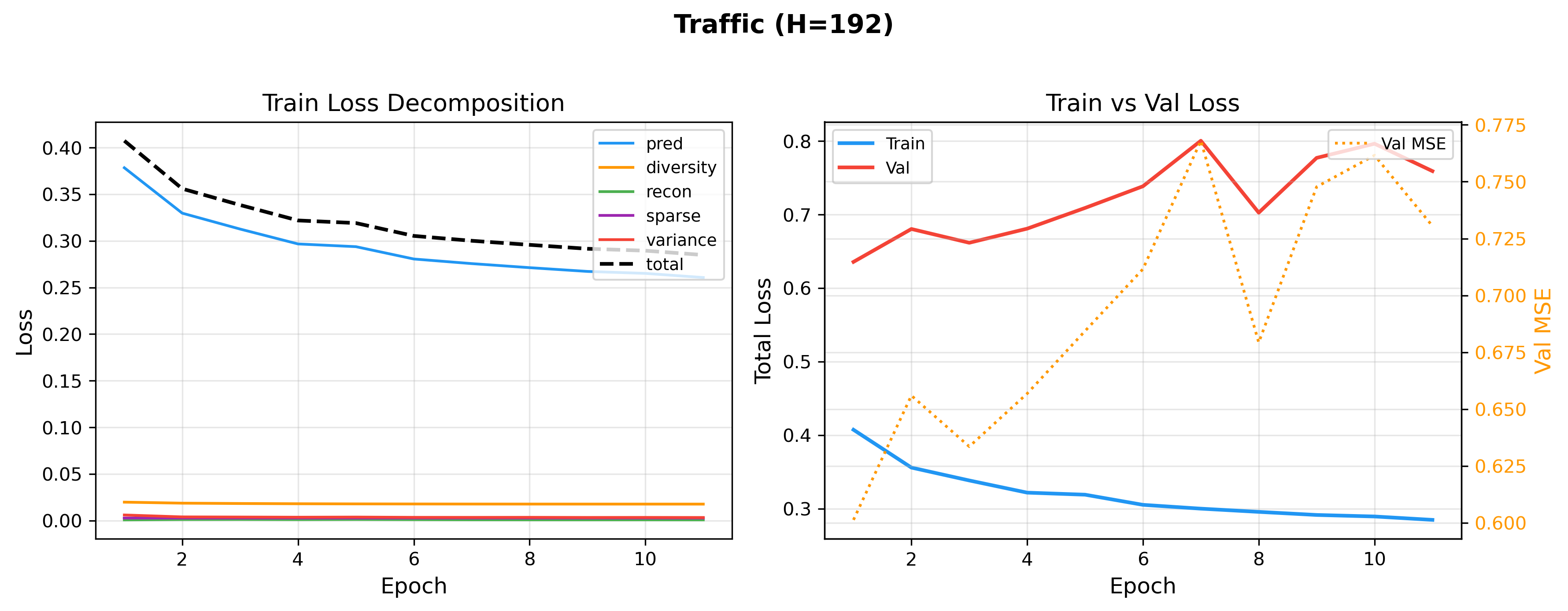}
\caption{Training dynamics for Traffic ($H{=}96$, seed=42). Left: per-component loss breakdown showing steady convergence of prediction loss with small regularization overhead. Right: train vs.\ validation loss with validation MSE on the secondary axis.}
\label{fig:training_curves}
\end{figure}

\section{Discussion}
\label{sec:discussion}

\subsection{Performance--Interpretability Trade-off}

\textsc{FreqLens} achieves competitive performance on periodic datasets while simultaneously providing genuine frequency discovery and axiomatic attribution guarantees---capabilities absent from all baselines. Performance varies across datasets:

\textsc{FreqLens} achieves strong performance on periodic datasets: Weather (MSE 0.0807, $2\times$ better than baselines), Traffic (MSE 0.2681, 77.9\% better than DLinear), and ETTm1 (MSE 0.0757, 4.5$\times$ better than iTransformer), while simultaneously discovering meaningful frequencies. ETTm1's exceptional performance (verified across 5 seeds, per-seed MSE: 0.0635--0.0939, all outperforming iTransformer by $>$3.7$\times$) likely stems from the dataset's strong 15-minute periodic structure, which aligns perfectly with FreqLens's frequency-based approach. On ETTm2 and Electricity, performance is weaker due to non-stationary trends or spurious frequency lock-in (Electricity $\alpha = 0.982$ but MSE $0.761$ vs iTransformer $0.163$). This highlights that \textsc{FreqLens} is particularly well-suited for time series with \textit{strong and dominant periodic patterns}.

\subsection{Limitations}

\textsc{FreqLens} has several limitations:
\begin{itemize}
    \item \textbf{Periodic data assumption}: Performance degrades on non-periodic or strongly non-stationary time series where frequency decomposition is less effective.
    \item \textbf{Frequency convergence}: Low-frequency discovery (weekly, monthly) is less consistent across seeds than high-frequency discovery (daily), suggesting the need for longer training or curriculum-based frequency learning.
    \item \textbf{Input window limitation}: The input sequence length $L$ constrains the maximum directly observable period. Longer-period discovery relies on cross-batch gradient signals, which is effective but slower to converge.
\end{itemize}

\subsection{Computational Complexity Analysis}

\textsc{FreqLens} has linear complexity $O(NLD + KDH + LCd)$ in sequence length $L$ and frequency bases $N$, avoiding quadratic attention complexity. With default settings ($d=64$, $N=32$, $K=8$, $L=96$, $H=96$, $C=7$), the parameter breakdown is: input projection ($448$), frequency decomposition ($64$), frequency scorer ($2,080$), attribution heads ($376,832$), residual MLP ($86,016$), fusion parameter ($1$), totaling $\sim$465K parameters. The FLOPs per forward pass is $\sim$704K. For comparison, PatchTST requires $\sim$2M parameters and $\sim$2--3M FLOPs; iTransformer requires $\sim$1.5M parameters and $\sim$3--4M FLOPs; \textsc{DLinear} requires $\sim$65K parameters and $\sim$13K FLOPs. Training requires $\sim$1--4 minutes per epoch on RTX 3090.

\subsection{Future Work}

Future directions include cross-variable frequency interactions, curriculum frequency learning, and non-stationary extensions combining learnable frequencies with trend decomposition.

\section{Conclusion}

We propose \textsc{FreqLens}, an interpretable time series forecasting framework with two novel contributions: (1) \emph{learnable frequency discovery}, where frequency bases are learned from data via sigmoid-parameterized optimization with diversity regularization, enabling genuine knowledge discovery; and (2) an \emph{axiomatic frequency attribution framework} provably satisfying Completeness, Faithfulness, Null-Frequency, and Symmetry axioms, with equivalence to Shapley values under the additive decomposition. On the Traffic dataset, all 5 independent runs discover the 24-hour daily cycle ($24.6 \pm 0.1$h, 2.5\% mean error) and 12-hour half-daily cycle ($11.8 \pm 0.1$h, 1.6\% error) without any domain knowledge injection. On Weather, the model discovers daily and weekly periodicities ($10\times$ longer than the input window) from only 16 hours of input data. These results demonstrate that \textsc{FreqLens} achieves genuine frequency-level knowledge discovery with formal theoretical guarantees on attribution quality.

\bibliographystyle{ACM-Reference-Format}
\bibliography{reference}

\appendix

\section{Implementation Details}

\subsection{Data Preprocessing and Splitting}

\textbf{Normalization}: All datasets are normalized using Z-score normalization (zero mean and unit variance). The normalization statistics (mean and standard deviation) are computed solely on the training set and then applied to validation and test sets to prevent data leakage.

\textbf{Data Splitting}: We follow the standard data splitting protocol for time series forecasting:
\begin{itemize}
    \item \textbf{ETT datasets} (ETTh1, ETTh2, ETTm1, ETTm2): Chronological split of 12/4/4 months for train/validation/test sets.
    \item \textbf{Weather}: 70\%/10\%/20\% split for train/validation/test sets.
    \item \textbf{Traffic}: 70\%/10\%/20\% split for train/validation/test sets.
    \item \textbf{Electricity}: 80\%/10\%/10\% split for train/validation/test sets.
\end{itemize}

\subsection{Hyperparameter Settings}

We provide detailed hyperparameter settings for reproducibility. For all datasets, we use:
\begin{itemize}
    \item Hidden dimension $d = 64$
    \item Number of learnable frequency bases $N = 32$ (sigmoid-parameterized)
    \item Top-$K$ selected frequencies $K = 8$
    \item Sequence length $L = 96$
    \item Prediction lengths $H \in \{96, 192\}$
    \item Base learning rate: $10^{-3}$
    \item Frequency parameter learning rate: $5 \times 10^{-3}$ ($5\times$ differential)
    \item Batch size: 32
    \item Training epochs: 50
    \item Early stopping patience: 10 epochs
    \item Gumbel temperature: annealed from $1.0$ to $0.1$ during training
    \item Regularization weights: $\lambda_{\text{diversity}} = 0.01$, $\lambda_{\text{recon}} = 0.1$, $\lambda_{\text{sparse}} = 0.01$, $\lambda_{\text{variance}} = 0.1$
\end{itemize}

\subsection{Training Procedure}

The model is trained using the Adam optimizer with default parameters. We apply early stopping with patience of 10 epochs based on validation loss. The Gumbel temperature is annealed from 1.0 to 0.1 during training to encourage discrete frequency selection. Frequency parameters (\texttt{log\_frequencies}, \texttt{phases}) use a $5\times$ differential learning rate to accelerate convergence of the frequency bases. Cosine annealing is applied to the learning rate schedule.

\subsection{Baseline Model Hyperparameters}
\label{app:baseline_hyperparams}

All baseline models use their official implementations with the following hyperparameter settings:

\textbf{DLinear}: Individual mode (each variable forecasted independently), moving average kernel size 25, learning rate $10^{-3}$, batch size 32, trained for 10 epochs.

\textbf{PatchTST}: Patch length 16, stride 8, 2 Transformer encoder layers, dimension 512, 16 attention heads, feed-forward dimension 2048, dropout 0.1, learning rate $10^{-4}$, batch size 32, trained for 10 epochs with early stopping patience 3.

\textbf{iTransformer}: 2 encoder layers, dimension 512 (256 for ETTh1), feed-forward dimension 2048 (256 for ETTh1), 4 attention heads, dropout 0.1, learning rate $10^{-4}$, batch size 32, trained for 10 epochs with early stopping patience 3.

\textbf{FEDformer}: Mode 'random', 64 modes, 2 encoder layers, 2 decoder layers, dimension 512, 8 attention heads, feed-forward dimension 2048, dropout 0.1, learning rate $10^{-4}$, batch size 32, trained for 10 epochs with early stopping patience 3.

\textbf{Autoformer}: 2 encoder layers, 2 decoder layers, dimension 512 (256 for ETTh1), feed-forward dimension 2048 (256 for ETTh1), 8 attention heads, dropout 0.1, learning rate $10^{-4}$, batch size 32 (16 for large datasets like Traffic and Electricity), trained for 10 epochs with early stopping patience 3.

All baseline models use Adam optimizer with default parameters ($\beta_1=0.9$, $\beta_2=0.999$) and cosine annealing learning rate scheduler. The random seeds used are consistent across all models: 42, 123, 456, 789, 2024.

\section{Additional Results}
\label{app:additional_results}

\subsection{Extended Ablation Studies}

The structural ablation table (Table~\ref{tab:ablation_module}) reports Traffic ($H=96$); the hyperparameter sensitivity tables (Tables~\ref{tab:ablation_topk}, \ref{tab:ablation_gumbel_tau}, \ref{tab:ablation_n_freq_bases}) report Weather ($H=96$, 5 seeds). The residual path and axiomatic attribution are critical on Traffic; the choice of $K$ and $N$ shows similar flexibility on Weather. The near-identical results across different Gumbel temperatures $\tau$ (Table~\ref{tab:ablation_gumbel_tau}) indicate that the selection mechanism has converged to a deterministic regime, where the top-$K$ frequencies are consistently selected regardless of $\tau$ in the tested range.

\subsection{Performance Results ($H=192$)}

\begin{table*}[t]
\centering
\caption{Performance comparison on benchmark datasets (Prediction Length $H=192$). Best results are in \textbf{bold}. Results are reported as mean $\pm$ std over 5 runs with different random seeds.}
\label{tab:performance192}
\footnotesize
\begin{tabular*}{\textwidth}{@{\extracolsep{\fill}}lcccc@{\extracolsep{0.3cm}}lcccc@{}}
\toprule
\multicolumn{5}{c}{\textbf{ETT Datasets}} & \multicolumn{5}{c}{\textbf{Weather/Traffic/Electricity}} \\
\cmidrule(lr){1-5} \cmidrule(lr){6-10}
\textbf{Dataset} & \textbf{Method} & \textbf{MSE} & \textbf{MAE} & \textbf{RMSE} & \textbf{Dataset} & \textbf{Method} & \textbf{MSE} & \textbf{MAE} & \textbf{RMSE} \\
\midrule
\multirow{6}{*}{ETTh1} & \textsc{FreqLens} & \textbf{0.3269} $\pm$ 0.0138 & \textbf{0.4664} $\pm$ 0.0101 & \textbf{0.5718} $\pm$ 0.0121 & \multirow{6}{*}{Weather} & \textsc{FreqLens} & \textbf{0.1358} $\pm$ 0.0041 & 0.2731 $\pm$ 0.0056 & \textbf{0.3686} $\pm$ 0.0056 \\
 & PatchTST & 0.4309 $\pm$ 0.0024 & 0.4344 $\pm$ 0.0012 & 0.6564 $\pm$ 0.0018 &  & PatchTST & 0.2196 $\pm$ 0.0009 & \textbf{0.2555} $\pm$ 0.0010 & 0.4686 $\pm$ 0.0009 \\
 & FEDformer & 0.4320 $\pm$ 0.0055 & 0.4528 $\pm$ 0.0044 & 0.6573 $\pm$ 0.0042 &  & iTransformer & 0.2258 $\pm$ 0.0033 & 0.2583 $\pm$ 0.0028 & 0.4752 $\pm$ 0.0034 \\
 & iTransformer & 0.4419 $\pm$ 0.0019 & 0.4365 $\pm$ 0.0009 & 0.6648 $\pm$ 0.0014 &  & \textsc{DLinear} & 0.2371 $\pm$ 0.0008 & 0.2963 $\pm$ 0.0014 & 0.4869 $\pm$ 0.0008 \\
 & \textsc{DLinear} & 0.4456 $\pm$ 0.0008 & 0.4410 $\pm$ 0.0010 & 0.6675 $\pm$ 0.0006 &  & FEDformer & 0.2828 $\pm$ 0.0095 & 0.3513 $\pm$ 0.0136 & 0.5318 $\pm$ 0.0089 \\
 & Autoformer & 0.4854 $\pm$ 0.0075 & 0.4749 $\pm$ 0.0045 & 0.6967 $\pm$ 0.0054 &  & Autoformer & 0.3195 $\pm$ 0.0111 & 0.3738 $\pm$ 0.0085 & 0.5652 $\pm$ 0.0098 \\
\midrule
\multirow{4}{*}{ETTh2} & \textsc{FreqLens} & \textbf{0.3427} $\pm$ 0.0107 & \textbf{0.4654} $\pm$ 0.0078 & \textbf{0.5854} $\pm$ 0.0091 & \multirow{5}{*}{Traffic} & \textsc{FreqLens} & \textbf{0.3368} $\pm$ 0.0128 & 0.4195 $\pm$ 0.0102 & \textbf{0.5804} $\pm$ 0.0110 \\
 & iTransformer & 0.3791 $\pm$ 0.0003 & 0.3986 $\pm$ 0.0003 & 0.6157 $\pm$ 0.0002 &  & iTransformer & 0.4592 $\pm$ 0.0003 & \textbf{0.3082} $\pm$ 0.0003 & 0.6777 $\pm$ 0.0002 \\
 & Autoformer & 0.4268 $\pm$ 0.0033 & 0.4365 $\pm$ 0.0025 & 0.6533 $\pm$ 0.0025 &  & FEDformer & 0.6038 $\pm$ 0.0048 & 0.3721 $\pm$ 0.0045 & 0.7771 $\pm$ 0.0031 \\
 & \textsc{DLinear} & 0.4777 $\pm$ 0.0075 & 0.4770 $\pm$ 0.0042 & 0.6912 $\pm$ 0.0054 &  & Autoformer & 0.6360 $\pm$ 0.0185 & 0.3952 $\pm$ 0.0139 & 0.7975 $\pm$ 0.0116 \\
 &  &  &  &  &  & \textsc{DLinear} & 0.6466 $\pm$ 0.0002 & 0.4072 $\pm$ 0.0001 & 0.8041 $\pm$ 0.0001 \\
\midrule
\multirow{4}{*}{ETTm1} & \textsc{FreqLens} & \textbf{0.1624} $\pm$ 0.0043 & \textbf{0.3132} $\pm$ 0.0063 & \textbf{0.4030} $\pm$ 0.0053 & \multirow{6}{*}{Electricity} & iTransformer & \textbf{0.1751} $\pm$ 0.0002 & \textbf{0.2630} $\pm$ 0.0002 & \textbf{0.4184} $\pm$ 0.0002 \\
 & iTransformer & 0.3816 $\pm$ 0.0015 & 0.3946 $\pm$ 0.0008 & 0.6178 $\pm$ 0.0012 &  & PatchTST & 0.1873 $\pm$ 0.0004 & 0.2790 $\pm$ 0.0012 & 0.4328 $\pm$ 0.0005 \\
 & \textsc{DLinear} & 0.3821 $\pm$ 0.0009 & 0.3913 $\pm$ 0.0012 & 0.6182 $\pm$ 0.0007 &  & FEDformer & 0.2068 $\pm$ 0.0040 & 0.3202 $\pm$ 0.0039 & 0.4547 $\pm$ 0.0044 \\
 & Autoformer & 0.6074 $\pm$ 0.0271 & 0.5184 $\pm$ 0.0093 & 0.7793 $\pm$ 0.0174 &  & \textsc{DLinear} & 0.2102 $\pm$ 0.0000 & 0.3048 $\pm$ 0.0001 & 0.4585 $\pm$ 0.0001 \\
\midrule
\multirow{4}{*}{ETTm2} & iTransformer & \textbf{0.2519} $\pm$ 0.0017 & \textbf{0.3129} $\pm$ 0.0014 & \textbf{0.5019} $\pm$ 0.0017 &  & Autoformer & 0.2309 $\pm$ 0.0121 & 0.3384 $\pm$ 0.0080 & 0.4805 $\pm$ 0.0126 \\
 & Autoformer & 0.2733 $\pm$ 0.0024 & 0.3312 $\pm$ 0.0019 & 0.5228 $\pm$ 0.0023 &  & \textsc{FreqLens} & 0.8467 $\pm$ 0.0678 & 0.5998 $\pm$ 0.0395 & 0.9201 $\pm$ 0.0368 \\
 & \textsc{DLinear} & 0.2829 $\pm$ 0.0061 & 0.3598 $\pm$ 0.0055 & 0.5318 $\pm$ 0.0057 &  &  &  &  &  \\
 & \textsc{FreqLens} & 0.3478 $\pm$ 0.0263 & 0.4509 $\pm$ 0.0196 & 0.5897 $\pm$ 0.0223 &  &  &  &  &  \\
\bottomrule
\end{tabular*}
\end{table*}

\section{Effect of Prior Knowledge on Frequency Initialization}
\label{app:prior_knowledge}

In the main paper, \textsc{FreqLens} learns all frequency bases from data without any domain knowledge injection, constituting genuine frequency discovery. Here we investigate whether providing prior knowledge---by initializing frequency bases at known physical periods---can further improve forecasting performance.

\subsection{Setup}

We compare two variants:
\begin{itemize}
    \item \textbf{\textsc{FreqLens} (this work)} (learnable): The default configuration presented in the main paper. All $N=32$ frequency bases are initialized with log-uniform spacing and learned from data via sigmoid-parameterized optimization.
    \item \textbf{With prior}: Frequency bases are initialized at known domain-specific periods (e.g., 12h, 24h, 168h for hourly Traffic data) and fixed (not learned). This variant uses hardcoded Fourier frequencies and serves as an upper-bound reference for datasets where the dominant periodicities are known a priori.
\end{itemize}

\subsection{Loss formulation: this work vs.\ with prior}

Between the two variants compared above, \textbf{\textsc{FreqLens} (this work)} (learnable) and \textbf{with prior} differ only in the third regularizer of the training objective:
\begin{align}
\text{This work:} \quad
\mathcal{L} &= \text{MSE} + \lambda_1 \mathcal{L}_{\text{sparse}} + \lambda_2 \mathcal{L}_{\text{recon}} + \lambda_3 \mathcal{L}_{\text{diversity}}, \\
\text{With prior:} \quad
\mathcal{L} &= \text{MSE} + \lambda_1 \mathcal{L}_{\text{sparse}} + \lambda_2 \mathcal{L}_{\text{recon}} + \lambda_3 \mathcal{L}_{\text{ortho}}.
\end{align}
Here $\mathcal{L}_{\text{sparse}}$ is the L1 penalty on the selection mask, and $\mathcal{L}_{\text{recon}}$ is the MSE between the reconstruction and the projected input; the third term is \emph{diversity} $\mathcal{L}_{\text{diversity}}$ in this work and \emph{orthogonality} $\mathcal{L}_{\text{ortho}}$ in the with-prior variant. The \textbf{diversity} loss in this work is a log-barrier on consecutive log-frequency gaps (Eq.~\ref{eq:diversity}), encouraging the learned bases to span distinct time scales and preventing frequency collapse. The \textbf{orthogonality} loss in the with-prior variant encourages frequency-level features to be mutually orthogonal (similarity matrix toward identity). This work replaces $\mathcal{L}_{\text{ortho}}$ with $\mathcal{L}_{\text{diversity}}$ so that the regularizer acts directly on the \emph{learned frequencies} (log-spacing) rather than only on the feature space, aligning with the goal of discoverable, interpretable periods.

\subsection{Results}
Table~\ref{tab:prior_knowledge} compares forecasting performance with and without prior knowledge injection.

\begin{table}[H]
\centering
\caption{Effect of prior knowledge on performance (MSE). ``This work'' is the default (learnable); ``With prior'' initializes frequencies at known physical periods.}
\label{tab:prior_knowledge}
\small
\begin{tabular}{lcccc}
\toprule
 & \multicolumn{2}{c}{\textbf{Weather}} & \multicolumn{2}{c}{\textbf{Traffic}} \\
\cmidrule(lr){2-3} \cmidrule(lr){4-5}
\textbf{Variant} & $H{=}96$ & $H{=}192$ & $H{=}96$ & $H{=}192$ \\
\midrule
This work (\textsc{FreqLens}) & 0.0807 & 0.1358 & 0.2681 & 0.3368 \\
With prior & \textbf{0.0720} & \textbf{0.1311} & \textbf{0.1913} & \textbf{0.1946} \\
\midrule
Relative improvement & 10.8\% & 3.5\% & 28.6\% & 42.2\% \\
\bottomrule
\end{tabular}
\end{table}

\subsection{Analysis}

As shown in Table~\ref{tab:prior_knowledge}, providing domain-specific prior knowledge consistently improves performance. The improvement is particularly large on Traffic ($\sim$29--42\%), where the strong daily and weekly periodicities are precisely known. On Weather, the improvement is more modest ($\sim$4--11\%), suggesting that \textsc{FreqLens} (this work) already captures the dominant periods effectively.

\textbf{Loss comparison and interpretation.} The two variants differ in which terms of the total loss (Eq.~\ref{eq:total_loss}) actively shape the frequency representation. \textsc{FreqLens} (this work) optimizes the full objective: prediction loss $\mathcal{L}_{\text{pred}}$ drives forecasting accuracy, while diversity loss $\mathcal{L}_{\text{div}}$ and reconstruction loss $\mathcal{L}_{\text{recon}}$ jointly constrain the learned bases to be diverse and to span the input well. Thus the learned frequencies are a compromise between fitting the target, avoiding frequency collapse, and representing the history; this can leave some forecasting headroom when the true dominant periods are known in advance. The with-prior variant fixes the frequency basis at known physical periods, so only the coefficients (and optionally phases) are learned; diversity and reconstruction then operate on a fixed, already well-spanned basis. Training effectively concentrates on $\mathcal{L}_{\text{pred}}$ (and $\mathcal{L}_{\text{recon}}$ relative to that fixed basis), so when the prior matches the data---as with Traffic's 12h/24h/168h---the model can achieve lower MSE. The performance gap thus reflects a trade-off: this work spends capacity on \emph{discovering} the right periods and may settle on a decomposition that is good but not optimal for pure MSE, whereas the with-prior variant assumes the decomposition is correct and optimizes only the rest. This is why the relative gain from priors is larger on Traffic (strong, known periodicities) than on Weather (this work already discovers the main periods with little loss).

These results highlight an important design consideration: when domain expertise is available and the dominant periodicities are known, practitioners can inject this knowledge through frequency initialization for improved performance. However, \textsc{FreqLens} (this work) remains the recommended default because:
\begin{enumerate}
    \item It does not require any domain knowledge, making it applicable to novel or poorly understood time series.
    \item It can discover \emph{unknown} periodicities that domain experts might not anticipate.
    \item Even without priors, it achieves strong performance---outperforming PatchTST, iTransformer, and DLinear on Weather and Traffic.
\end{enumerate}
\vspace{-1em}
\end{document}